\documentclass[sigconf]{acmart}

\AtBeginDocument{%
  \providecommand\BibTeX{{%
    \normalfont B\kern-0.5em{\scshape i\kern-0.25em b}\kern-0.8em\TeX}}}

\copyrightyear{2021} 
\acmYear{2021} 
\setcopyright{acmcopyright}\acmConference[KDD '21]{Proceedings of the 27th ACM SIGKDD Conference on Knowledge Discovery and Data Mining}{August 14--18, 2021}{Virtual Event, Singapore}
\acmBooktitle{Proceedings of the 27th ACM SIGKDD Conference on Knowledge Discovery and Data Mining (KDD '21), August 14--18, 2021, Virtual Event, Singapore}
\acmPrice{15.00}
\acmDOI{10.1145/3447548.3467267}
\acmISBN{978-1-4503-8332-5/21/08}
\usepackage{epsfig}
\usepackage{graphicx}
\usepackage{amsmath}
\usepackage{algorithm}
\usepackage{algorithmic}
\usepackage{mathrsfs}
\usepackage{multirow}
\usepackage{amsthm}
\usepackage{bm}
\usepackage{bbold}
\usepackage{array}
\newtheorem{theorem}{Theorem}
\newtheorem{lemma}{Lemma}
\newtheorem{proposition}{Proposition}

\newtheorem{assumption}{Assumption}

\usepackage{scalerel}
\usepackage{threeparttable}
\usepackage{makecell}
\usepackage{enumitem}
\allowdisplaybreaks[4]
\title{Large-Scale Subspace Clustering via k-Factorization}

\author{Jicong Fan}
\email{fanjicong@cuhk.edu.cn}
\orcid{0000-0001-9665-0355}
\affiliation{%
  \institution{The Chinese University of Hong Kong (Shenzhen) and Shenzhen Research Institute of Big Data}
  \city{Shenzhen}
  \country{China}
}

\renewcommand{\shortauthors}{Jicong Fan}

\begin{document}
\fancyhead{}

\begin{abstract}
Subspace clustering (SC) aims to cluster data lying in a union of low-dimensional subspaces. Usually, SC learns an affinity matrix and then performs spectral clustering. Both steps suffer from high time and space complexity, which leads to difficulty in clustering large datasets. This paper presents a method called k-Factorization Subspace Clustering (k-FSC) for large-scale subspace clustering. K-FSC directly factorizes the data  into k groups via pursuing structured sparsity in the matrix factorization model. Thus, k-FSC avoids learning affinity matrix and performing eigenvalue decomposition, and has low (linear) time and space complexity on large datasets. This paper proves the effectiveness of the k-FSC model theoretically. An efficient algorithm with convergence guarantee is proposed to solve the optimization of k-FSC. In addition, k-FSC is able to handle sparse noise, outliers, and missing data, which are pervasive in real applications. This paper also provides online extension and out-of-sample extension for k-FSC to handle streaming data and cluster arbitrarily large datasets. Extensive experiments on large-scale real datasets show that k-FSC and its extensions outperform state-of-the-art methods of subspace clustering.
\end{abstract}

\ccsdesc[500]{Computing methodologies~Cluster analysis}
\ccsdesc[500]{Information systems~Clustering}

\keywords{Large-scale clustering; Subspace clustering; Spectral clustering; Matrix factorization}

\maketitle

\section{Introduction}
Subspace clustering \cite{sc2004, von2007tutorial,SC_tutorial2011,sim2013survey} assumes that the data points lie in a union of low-dimensional subspaces and segments the data into different groups corresponding to different subspaces. Classical subspace clustering algorithms such as sparse subspace clustering (SSC) \cite{SSC_PAMIN_2013}, low-rank representation (LRR) \cite{LRR_PAMI_2013},  and their variants \cite{li2015structured,lu2018subspace,zhang2019neural,9366807} are based on the self-expressive property \cite{SSC_PAMIN_2013}.
These algorithms\footnote{This paper focuses on a relatively narrow definition of and approach to subspace clustering. More general problems and methods of subspace clustering can be found in \cite{sim2013survey}. } have two major steps. First, they learn an affinity matrix, of which the time complexity is $O(n^2)$ or even $O(n^3)$ in every iteration of the optimization. The second step is to perform spectral clustering on the affinity matrix, of which the eigenvalue decomposition of the Laplacian matrix has polynomial time complexity. As a result, these algorithms are not applicable to large-scale data \cite{chen2010parallel} owing to their high time and space complexity.

Recently, a few fast subspace or spectral clustering algorithms were proposed for large-scale datasets \cite{fowlkes2004spectral,chen2011large,peng2013scalable,wang2014exact,you2016oracle,you2016scalable,li2017large,9292470,FAN201839,wu2018scalable,chen2020stochastic}. For instance, \cite{chen2011large} proposed a landmark-based spectral clustering algorithm: 1) produces a few landmark points using k-means; 2) computes the features of all data points via exploiting eigenvectors of the affinity matrix obtained from the landmark points; 3) performs k-means on the features to get the clusters. In \cite{peng2013scalable}, the authors treated large-scale subspace clustering as an out-of-sample extension problem of SSC on a few selected  landmark data points, in which the clustering problem on the remainders is solved via classification. 
In \cite{chen2018spectral}, the authors proposed an algorithm to directly optimize the normalized cut model and used an anchor-based strategy to extend the algorithm to large-scale data.
In \cite{matsushima2019selective}, a method called S$^5$C was proposed as a scalable variant of SSC. The method first selects a small subset of the data points by performing sparse representation iteratively; then it performs sparse representation again for all data points using the selected samples and constructs an affinity matrix for all data points; lastly, it uses orthogonal iteration to compute the required eigenvectors for clustering.

Although the aforementioned large-scale subspace clustering methods have achieved considerable success in numerous applications, they still have a few limitations. First, those methods often start with a few samples of the dataset and then expand the representation coefficients to all data. Thus the clustering cannot effectively exploit the whole information of the dataset, which may reduce the clustering accuracy. Second, those methods usually have to store the affinity matrix and compute eigenvectors, which prevent the application to extremely large datasets. Finally, those methods are not effective in handling sparse noise, outliers, missing data, and streaming data, which are pervasive in real applications. 

To handle the aforementioned problems, this paper presents a method called k-Factorization Subspace Clustering (k-FSC), which directly factorizes the data matrix into $k$ groups corresponding to $k$ subspaces. The contributions of this work are as follows. 

\textbf{(1)}  The paper proposes a group-sparse factorization model for subspace clustering. The method k-FSC does not need to learn an affinity matrix and perform spectral clustering. The time and space complexity of the method are linear with the number of data points.

\textbf{(2)} The paper provides theoretical guarantees for the effectiveness of the k-FSC model. 

\textbf{(3)} The paper provides an efficient algorithm with convergence guarantee for the nonconvex nonsmooth optimization of k-FSC.

\textbf{(4)} The paper provides online extension and out-of-sample extension for k-FSC to handle arbitrarily large datasets.

\textbf{(5)} The paper extends k-FSC to robust clustering that is able to handle sparse noise, outliers, and missing data.

Extensive experiments show that the proposed methods\footnote{The MATLAB codes of the proposed methods are available at \url{https://github.com/jicongfan/K-Factorization-Subspace-Clustering}.} outperform the state-of-the-art methods of large-scale clustering. 

The remainder of this paper is structured as follows. Section \ref{sec_2} elaborates the proposed k-FSC method. Section \ref{sec_3} is the optimization. Section \ref{sec_4} provides a few extensions of k-FSC. Section \ref{sec_5} discusses the connection with previous work. Section \ref{sec_exp} details the experiments. Section \ref{sec_con} presents the conclusion of this paper.

\section{k-Factorization Subspace Clustering} \label{sec_2}

Throughout the paper, we use the following notations. \ $\bm{x}$: column vector. $\bm{X}$: matrix. $[k]$: $\lbrace 1,2,\ldots,k\rbrace$.
\ $\bm{X}_{:j}$: column-$j$ of $\bm{X}$. \ $\bm{X}_j$: matrix with index $j$. 
\ $[\bm{X},\bm{Y}]$: {column-wise stack}. \ $[\bm{X};\bm{Y}]$: { row-wise stack}.
\ $\bm{X}^{(j)} $: index-$j$ sub-matrix of $\bm{X}$.
\ $\Vert \cdot\Vert$: Euclidean norm of vector.
\ $\Vert \cdot\Vert_2$: spectral norm of matrix.
\ $\Vert \cdot\Vert_F$: Frobenius norm of matrix.
\ $\Vert \cdot\Vert_1$: $\ell_1$ norm of vector or matrix.
\ $\vert \cdot\vert$: absolute value of scalar, vector, or matrix.
\ $\mathbb{1}(f)$: $1$ if $f$ is true; $0$ if $f$ is false.

We first give the following assumption.
\begin{assumption}\label{asump_X}
The columns of data matrix $\bm{X}\in\mathbb{R}^{m\times n}$ are drawn from a union of  $k$ low-dimensional subspaces: $\bm{x}_i=\bm{U}^{(j)}\bm{z}_i$ if $\bm{x}_i\in\mathcal{S}_j$, where ${\bm{U}^{(j)}}\in\mathbb{R}^{m\times d_j}$ are the orthogonal bases of  $\mathcal{S}_j$,  $j\in [k]$, and $i\in[n]$. The number of data points lying in $\mathcal{S}_j$ is $n_j$ and $d_j<\min\lbrace m, n_j\rbrace$.
\end{assumption}
Our goal is to perform subspace clustering on $\bm{X}$ given by Assumption \ref{asump_X}.
In contrast to conventional subspace clustering methods, we in this paper propose to directly factorize $\bm{X}$ into $k$ groups corresponding to $k$ subspaces. Intuitively, we want to solve
\begin{equation}\label{eq.ksc_0}
\begin{aligned}
\mathop{\textup{minimize}}_{\bm{P}, \bm{U},\bm{Z}}&\ \ \Vert \bm{X}\bm{P}-\bm{U}\bm{Z}\Vert_F^2,\\
\end{aligned}
\end{equation}
where $\bm{P}\in\mathbb{R}^{n\times n}$ is a permutation matrix,  $\bm{U}=[\bm{U}^{(1)},\ldots,\bm{U}^{(k)}]\in\mathbb{R}^{m\times \sum _{j=1}^kd_j}$ are the subspace bases, $\bm{U}^{(j)}\in\mathbb{R}^{m\times d_j}$,  and ${\bm{U}^{(j)}}^\top \bm{U}^{(j)}=\bm{I}_{d_j}$, for $j\in [k]$. $\bm{Z}$ is a block diagonal matrix, i.e.
$$ \bm{Z}=\left[\begin{matrix}
\bm{Z}^{(1)} & \ldots &\bm{0}\\
\vdots & \ddots &\vdots\\
\bm{0} & \ldots &\bm{Z}^{(k)}
\end{matrix}
\right]\in\mathbb{R}^{\sum_{j=1}^d d_j\times n},
$$
where $\bm{Z}^{(j)}\in\mathbb{R}^{d_j\times n_j}$, for $j\in [k]$. 	The minimum of the objective function in \eqref{eq.ksc_0} is 0. In fact, it is difficult to solve \eqref{eq.ksc_0} directly because of the presence of $\bm{P}$. 

Notice that in \eqref{eq.ksc_0} we can replace $\bm{U}$ with $\bm{D}\in\mathbb{S}_D$, where 
\begin{align*}
\mathbb{S}_D:=&\lbrace [\bm{D}^{(1)},\ldots,\bm{D}^{(k)}]\in\mathbb{R}^{m\times \sum _{j=1}^kd_j}: \bm{D}^{(j)}\in\mathbb{R}^{m\times d_j},\\
&\ \Vert \bm{D}^{(j)}_{:i}\Vert\leq 1, \forall\ i\in[d_j], j\in[k]\rbrace.
\end{align*}
Meanwhile, we replace $\bm{X}\bm{P}$ with $\bm{X}$ and let $\bm{C}=\bm{Z}\bm{P}^{-1}$ $\in\mathbb{S}_C$ where
\begin{align*}
\mathbb{S}_C&:=\lbrace [\bm{C}^{(1)};\ldots;\bm{C}^{(k)}]\in\mathbb{R}^{\sum_{j=1}^d d_j\times n}: \bm{C}^{(j)}\in\mathbb{R}^{d_j\times n},\\
&\ \sum_{i=1}^n\mathbb{1}(\bm{C}_{:i}^{(j)}\neq \bm{0})=n_j, \forall\ j\in [k];\sum_{j=1}^k\mathbb{1}(\bm{C}_{:i}^{(j)}\neq \bm{0})=1, \forall\ i\in [n]\rbrace.
\end{align*}
Namely, $\bm{C}$ is a sparse matrix, $\bm{C}^{(j)}$ has $n_j$ nonzero columns, and the number of nonzero groups in each column of $\bm{C}$ is 1.
Thus we see that we actually don't need to determine $\bm{P}$ explicitly. Instead, we merge $\bm{P}$ into $\bm{C}$, which yields the following problem
\begin{equation}\label{eq.ksc_1}
\begin{aligned}
\mathop{\textup{minimize}}_{\bm{D}\in\mathbb{S}_D,\bm{C}\in\mathbb{S}_C}&\ \ \Vert \bm{X}-\bm{D}\bm{C}\Vert_F^2.
\end{aligned}
\end{equation}
Once $\bm{C}$ is obtained from \eqref{eq.ksc_1}, the clusters can be identified as
\begin{equation}\label{eq.cluster_assign_0}
\bm{x}_i\in c_j,\quad  \textup{if}\ \bm{C}_{:i}^{(j)}\neq \bm{0},\ i\in[n],
\end{equation}
where $c_j$ corresponds to $\mathcal{S}_j$, $j\in [k]$. The bases of $\mathcal{S}_1,\dots,\mathcal{S}_k$ can be computed by applying singular value decomposition to the dictionaries $\bm{D}^{(1)}, \ldots,\bm{D}^{(k)}$. We call \eqref{eq.ksc_1} k-Factorization Subspace Clustering (k-FSC). The general idea of k-FSC is shown in Figure \ref{fig_kFSC}.

\begin{figure}[h!]
\centering
\includegraphics[width=8.5cm]{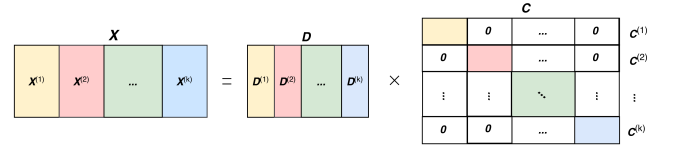}
\caption{Model structure of k-FSC (we let $\bm{X}=[\bm{X}^{(1)},\ldots,\bm{X}^{(k)}]$ for better visualization, though in practice $\bm{X}$ has been permuted before clustering and $\bm{C}$ is not  block-diagonal; but there always exists a permutation making $\bm{C}$ block-diagonal).}\label{fig_kFSC}
\end{figure}

In \eqref{eq.ksc_1}, the constraints on $\bm{C}$ make the problem difficult to solve. Particularly, in $\mathbb{S}_C$,  the first constraint on $\bm{C}$ may never be ensured because we usually don't know the number of data points in each clusters. However, later, we will show that, without the first constraint on $\bm{C}$, we can still cluster the data correctly. To tackle the second constraint on $\bm{C}$, we propose and expect to minimize the number of nonzero groups in each column of $\bm{C}$ to 1 via solving
\begin{equation}\label{eq.ksc_1.5}
\begin{aligned}
\mathop{\textup{minimize}}_{\bm{D}\in\mathbb{S}_D,\bm{C}}&\ \ \sum_{i=1}^n\Vert[\bm{C}_{:i}^{(1)},\ldots, \bm{C}_{:i}^{(k)}]\Vert_{2,0}=\sum_{j=1}^k\Vert\bm{C}^{(j)}\Vert_{2,0},\\
\textup{subject to}& \ \ \bm{D}\bm{C}=\bm{X},
\end{aligned}
\end{equation}
where $\Vert \cdot\Vert_{2,0}$ denotes the number of nonzero columns of matrix. The following proposition (we defer all proof of this paper to the appendices) verifies the effectiveness of  \eqref{eq.ksc_1.5}.
\begin{proposition}\label{proposition_1}
Suppose $\bm{X}$ is given by Assumption \ref{asump_X}, where $d_1=\cdots d_k=d$ and $\min_{j\in[d],i\in[n]}\vert z_{ji}\vert>0$. In \eqref{eq.ksc_1.5}, let $\bm{D}^{(j)}\in\mathbb{R}^{m\times \hat{d}}$, $\forall j\in[k]$, and let $\bm{C}^\ast$ be an optimal solution.\\
(a)  If the subspaces are independent, i.e. $\textup{dim}(\mathcal{S}_1\cup\mathcal{S}_1\cdots\cup\mathcal{S}_k)=\sum_{j=1}^k\textup{dim}(\mathcal{S}_j)=kd$, and $d\leq \hat{d}<2d$, applying \eqref{eq.cluster_assign_0} to $\bm{C}^\ast$ clusters the data correctly.
\\
(b) If the subspaces share $\bar{d}$ bases ($\bar{d}<d$), i.e. $\textup{dim}(\mathcal{S}_1\cup\mathcal{S}_1\cdots\cup\mathcal{S}_k)=\sum_{j=1}^k(\textup{dim}(\mathcal{S}_j)-\bar{d})+\bar{d}=kd-(k-1)\bar{d}$, and $d\leq \hat{d}<2d-\bar{d}$, applying \eqref{eq.cluster_assign_0} to $\bm{C}^\ast$ clusters the data correctly.
\end{proposition}

Nevertheless, it is NP-hard to solve \eqref{eq.ksc_1.5} because of the presence of the $\ell_{2,0}$ norm. We define 
$$\mathbb{S}_{D}^+:=\lbrace \bm{D}: \bm{D}\in\mathbb{S}_{D}; \Vert{\bm{D}^{(l)}}^\dagger\bm{D}^{(j)}\Vert_2<1, \forall\ (l,j)\in[k]\times[k], l\neq j\rbrace,$$
where ${{\bm{D}}^{(l)}}^\dagger$ denotes the Moore–Penrose inverse (detailed in Appendix \ref{app_th1}) of $\hat{\bm{D}}^{(l)}$.
We can solve the following tractable problem 
\begin{equation}\label{eq.ksc_2_0}
\mathop{\textup{minimize}}_{\bm{D}\in\mathbb{S}_D^+,\bm{C}}\ \ \sum_{j=1}^k\Vert\bm{C}^{(j)}\Vert_{2,1},\quad \textup{subject to} \ \ \bm{D}\bm{C}=\bm{X},
\end{equation}
where $\Vert\cdot\Vert_{2,1}$ denotes the $\ell_{2,1}$ norm \cite{ding2006r} of matrix defined by $
\Vert\bm{Y}\Vert_{2,1}=\sum_j\Vert \bm{Y}_{:j}\Vert$.
It is a convex relaxation of the $\ell_{2,0}$ norm and has been used in many problems such as feature selection \cite{nie2010efficient} and matrix recovery \cite{fan2019factor}. 
We have the following theoretical guarantee.
\begin{theorem}\label{the_d}
Suppose $\bm{X}$ is given by Assumption \ref{asump_X}, where $d_1=\cdots d_k=d$ and $\min_{j\in[d],i\in[n]}\vert z_{ji}\vert>0$. In \eqref{eq.ksc_2_0}, let $\bm{D}^{(j)}\in\mathbb{R}^{m\times \hat{d}}$, $\forall j\in[k]$ and $d\leq\hat{d}<\infty$. Let $\bm{C}^\ast$ be an optimal solution. Then applying \eqref{eq.cluster_assign_0} to $\bm{C}^\ast$ clusters the data correctly.
\end{theorem}
%
Since real data are often noisy, we relax \eqref{eq.ksc_2_0} to
\begin{equation}\label{eq.ksc_2}
\begin{aligned}
\mathop{\textup{minimize}}_{\bm{D}\in\mathbb{S}_D,\bm{C}}&\ \ \dfrac{1}{2}\Vert \bm{X}-\bm{D}\bm{C}\Vert_F^2+\lambda\sum_{j=1}^k\Vert\bm{C}^{(j)}\Vert_{2,1},
\end{aligned}
\end{equation}
where $\lambda$ is a hyper-parameter to be determined in advance. Note that here we used $\mathbb{S}_D$ instead of $\mathbb{S}_D^+$ because the former is easier to optimize\footnote{As $\Vert{\bm{D}^{(l)}}^\dagger\bm{D}^{(j)}\Vert_2\leq\tfrac{\Vert{\bm{D}^{(l)}}^\top\bm{D}^{(j)}\Vert_2}{\sigma_{+_{\text{min}}}^{2}(\bm{D}^{(l)})}\leq\tfrac{\Vert{\bm{D}^{(l)}}^\top\bm{D}^{(j)}\Vert_F}{\sigma_{+_{\text{min}}}^{2}(\bm{D}^{(l)})}$, one may maximize the smallest nonzero singular value of $\bm{D}^{(l)}$ and minimize $\Vert{\bm{D}^{(l)}}^\top\bm{D}^{(j)}\Vert_2$ or $\Vert{\bm{D}^{(l)}}^\top\bm{D}^{(j)}\Vert_F$.}. In experiments, we observed that when using $\mathbb{S}_D$,  the constraint with $\mathbb{S}_D^+$ is often fulfilled implicitly provided that the angles between pair-wise subspaces are not too small.
The following theorem is able to provide a rule of thumb to set $\lambda$.
\begin{theorem}\label{the_lambda}
Suppose $\lbrace \bm{C}_\ast,\bm{D}_\ast\rbrace$ is a solution of \eqref{eq.ksc_2}.
For $i\in[n]$, let $\pi_{i1}=\mathop{\text{argmax}}_{1\leq j\leq k}\Vert {{\bm{D}}_\ast^{(j)}}^\top\bm{x}_i\Vert$ 
and 
$\pi_{i2}=\mathop{\text{argmax}}_{1\leq j\neq\pi_{i1}\leq k}\Vert {\bm{D}_\ast^{(j)}}^\top\bm{x}_i\Vert.$ If $\max_{1\leq i\leq n}\Vert {\bm{D}_\ast^{(\pi_{i2})}}^\top\bm{x}_i\Vert<\lambda\leq\min_{1\leq i\leq n}\Vert {\bm{D}_\ast^{(\pi_{i1})}}^\top\bm{x}_i\Vert$, then
$$\sum_{j=1}^k\mathbb{1}({\bm{C}_\ast}_{:i}^{(j)}\neq \bm{0})=1, \forall\ i\in [n].$$
\end{theorem}

According to Theorem \ref{the_lambda}, if we have a good initialization of $\bm{D}$, denoted by $\bm{D}_0$ (detailed in Section \ref{sec_initialization}), we can determine $\lambda$ as
\begin{equation}\label{eq.determine_lambda}
\lambda=\left(\max_{1\leq i\leq n}\Vert {\bm{D}_0^{(\pi_{i2})}}^\top\bm{x}_i\Vert+\min_{1\leq i\leq n}\Vert {\bm{D}_0^{(\pi_{i1})}}^\top\bm{x}_i\Vert\right)/2.
\end{equation}

Owning to noise and local minima, it is possible that $\sum_{j=1}^k\mathbb{1}(\bm{C}_{:i}^{(j)}\neq \bm{0})> 1$ for some $i\in[n]$ when we estimate $\bm{D}$ and $\bm{C}$ by \eqref{eq.ksc_2}. Thus, we cannot use \eqref{eq.cluster_assign_0} to assign the data into clusters. We propose to assign $\bm{x}_i$ to cluster $j$ if the reconstruction error given by $\bm{D}^{(j)}$ is the least:
\begin{equation}\label{eq.cluster_assign2}
\bm{x}_i\in c_j,\quad  j=\textup{argmin}_j\ \Vert \bm{x}_i-\bm{D}^{(j)}\hat{\bm{C}}_{:i}^{(j)}\Vert^2,\ i\in[n],
\end{equation}
where $\hat{\bm{C}}^{(j)}=({\bm{D}^{(j)}}^\top\bm{D}^{(j)}+\lambda'\bm{I})^{-1}{\bm{D}^{(j)}}^\top\bm{X}$ and $\lambda'$ is a small constant e.g. $10^{-5}$.

In practice, it is difficult to know $d_1,\ldots,d_k$ beforehand. We set $d_1=\ldots=d_k=d$, where $d$ is a relatively large number, though it can be arbitrarily large according to Theorem \ref{the_d}. Figure \ref{fig_syn_noise_d} in Section \ref{sec_exp_syn} and Figure \ref{Fig_syn_DL} in the Appendix \ref{app_DL} will show that k-FSC is not sensitive to $d$ and  indeed $d$ can be arbitrarily large. Comparing these results with Proposition \ref{proposition_1}, we see that \eqref{eq.ksc_2_0} and \eqref{eq.ksc_2} are much more flexible than \eqref{eq.ksc_1.5}  in terms of determining ${d}$ though they are the relaxed formulations of \eqref{eq.ksc_1.5}.

\section{Optimization for k-FSC}\label{sec_3}

Problem \eqref{eq.ksc_2} is nonconvex and nonsmooth. When $\bm{D}$ (or $\bm{C}$) is fixed, problem \eqref{eq.ksc_2} regarding of $\bm{C}$ (or $\bm{D}$) is convex. Hence we  update $\bm{D}$ and $\bm{C}$ alternately. 

\subsection{Initialization}\label{sec_initialization}
We can initialize $\bm{D}$ randomly, e.g. draw the entries of $\bm{D}$ from $\mathcal{N}(0,1)$. Alternatively,  we initialize $\bm{D}$ by k-means, which may improve the convergence of optimization and clustering accuracy. 
It is worth mentioning that k-means with Euclidean distance measure cannot exploit subspace information and hence does not give us an effective initialization. Instead, we use cosine similarity,  $\cos \theta=\tfrac{\bm{x}^\top\bm{y}}{\Vert\bm{x}\Vert\Vert\bm{y}\Vert}$, as a distance measure in k-means. Two data points lying in a same subspace tend to have larger absolute cosine value than lying in different subspaces. Therefore, k-means with cosine ``distance" measure is able to provide a better initialization for $\bm{D}$ than k-means with Euclidean distance measure. The procedures are: 1) perform k-means with cosine ``distance" measure on $\bm{X}$ (or a subset of $\bm{X}$ when the dataset is too large) to generate cluster centers $\bm{c}_1,\ldots,\bm{c}_k$;
2) for $j\in[k]$, let $\bm{D}_0^{(j)}$ consists of the left singular vectors of a matrix formed by  the $d$ columns of $\bm{X}$ (or the subset) closest to $\bm{c}_j$. Consequently, we initialize $\bm{C}$ by $\bm{C}_0=(\bm{D}_0^\top \bm{D}_0+\hat{\lambda}\bm{I})^{-1}\bm{D}_0^\top \bm{X}$, where $\hat{\lambda}$ is a small constant such as $10^{-5}$.

\subsection{Update ${C}$}\label{sec_updateC}

At iteration $t$, we fix $\bm{D}$ and solve
\begin{equation}\label{eq.ksc_C}
\begin{aligned}
\mathop{\textup{minimize}}_{\bm{C}}&\ \ \dfrac{1}{2}\Vert \bm{X}-\bm{D}_{t-1}\bm{C}\Vert_F^2+\lambda\sum_{j=1}^k\Vert\bm{C}^{(j)}\Vert_{2,1}.
\end{aligned}
\end{equation}
Decomposing \eqref{eq.ksc_C}, for $j\in [k]$, we aim to solve
\begin{equation}\label{eq.ksc_CC}
\begin{aligned}
\mathop{\textup{minimize}}_{\bm{C}^{(j)}}&\ \ \dfrac{1}{2}\Vert \bm{X}-\bar{\bm{X}}_j-\bm{D}_{t-1}^{(j)}\bm{C}^{(j)}\Vert_F^2+\lambda\Vert\bm{C}^{(j)}\Vert_{2,1},
\end{aligned}
\end{equation}
where $\bar{\bm{X}}_j=\sum_{l\neq j}\bm{D}_{t-1}^{(l)}\bm{C}_{t'}^{(l)}$, $t'=t$ if $l<j$, and $t'=t-1$ if $l>j$.  Problem \eqref{eq.ksc_CC} has no closed-form solution. Denote 
$$\mathcal{L}(\bm{C}^{(j)}):=\dfrac{1}{2}\Vert \bm{X}-\bar{\bm{X}}_j-\bm{D}_{t-1}^{(j)}\bm{C}^{(j)}\Vert_F^2.$$
The first order approximation of $\mathcal{L}(\bm{C}^{(j)})$ at $\bm{C}_{t-1}^{(j)}$ is
\begin{align*}
\hat{\mathcal{L}}(\bm{C}^{(j)}):=&\dfrac{1}{2}\Vert \bm{X}-\bar{\bm{X}}_j-\bm{D}_{t-1}^{(j)}\bm{C}_{t-1}^{(j)}\Vert_F^2\\
&+\left\langle \bm{C}^{(j)}-\bm{C}_{t-1}^{(j)},\bm{G}^{(j)}_{t-1}\right\rangle+\dfrac{\tau}{2}\Vert \bm{C}^{(j)}-\bm{C}_{t-1}^{(j)}\Vert_F^2,
\end{align*}
where $\bm{G}^{(j)}_{t-1}=\nabla_{\bm{C}^{(j)}}\mathcal{L}(\bm{C}^{(j)})=-{\bm{D}_{t-1}^{(j)}}^\top(\bm{X}-\bar{\bm{X}}_j-\bm{D}_{t-1}^{(j)}\bm{C}_{t-1}^{(j)})$ and $\tau\geq L_{j,t}:=\Vert \bm{D}_{t-1}^{(j)} \Vert_2^2$. As $\hat{\mathcal{L}}(\bm{C}^{(j)})\geq\mathcal{L}(\bm{C}^{(j)})$, we now minimize $\hat{\mathcal{L}}(\bm{C}^{(j)})+\lambda\Vert\bm{C}^{(j)}\Vert_{2,1}$, which is equivalent to
\begin{equation}\label{eq.ksc_CCC}
\begin{aligned}
\mathop{\textup{minimize}}_{\bm{C}^{(j)}}&\ \ \dfrac{\tau}{2}\Vert \bm{C}^{(j)}-\bm{C}_{t-1}^{(j)}+\tau^{-1}\bm{G}^{(j)}_{t-1}\Vert_F^2+\lambda\Vert\bm{C}^{(j)}\Vert_{2,1}.
\end{aligned}
\end{equation}

The closed-form solution of \eqref{eq.ksc_CCC} is
$$ \bm{C}_{t}^{(j)}=\Theta_{\lambda/\tau}(\bm{C}_{t-1}^{(j)}-\tau^{-1}\bm{G}^{(j)}_{t-1}),$$
where $\Theta_u(\cdot)$ is the column-wise soft-thresholding operator \cite{LRR_PAMI_2013}
\begin{equation}
\Theta_u(\bm{v})=\left\{
\begin{array}{ll}
\tfrac{(\Vert\bm{v}\Vert-u)\bm{v}}{\Vert\bm{v}\Vert}, &\textup{if}\ \Vert\bm{v}\Vert>u;\\
\bm{0}, &\textup{otherwise.}
\end{array}
\right.
\end{equation}

The update strategy for $\bm{C}$ is actually block coordinate descent in the manner of Gauss-Seidel. We further use extrapolation to accelerate the optimization \cite{xu2013block}. The procedures are summarized into Algorithm \ref{alg.GS}, in which fixing $\eta=0$ will remove the extrapolation. 

\renewcommand{\algorithmicrequire}{\textbf{Input:}}
\renewcommand{\algorithmicensure}{\textbf{Output:}}
\begin{algorithm}[h]
\caption{Update $\bm{C}$ by Gauss-Seidel method}
\label{alg.GS}
\begin{algorithmic}[1]
\REQUIRE
$\bm{X}$, $\bm{D}_{t-1}$, $\bm{C}_{t-1}$, $\bm{\Delta}$,  $\delta<1$, $\gamma\geq1$.
\FOR{$j=1,2,\ldots,k$}
\IF{$t\leq 2$}
\STATE $\eta_{j,t-1}=0$.
\ELSE
\STATE $\eta_{j,t-1}=\delta\sqrt{\tfrac{\tau_{j,t-2}}{\tau_{j,t-1}}}$.
\ENDIF
\STATE $\hat{\bm{C}}_{t-1}^{(j)}=\bm{C}_{t-1}^{(j)}-\eta_{j,t-1}\bm{\Delta}^{(j)}$.
\ENDFOR
\STATE$\hat{\bm{X}}=\bm{D}_{t-1}\hat{\bm{C}}_{t-1}$.
\FOR{$j=1,2,\ldots,k$}
\STATE $\bm{G}^{(j)}_{t-1}=-{\bm{D}^{(j)}_{t-1}}^\top(\bm{X}-\hat{\bm{X}})$.
\STATE $\tau_{j,t-1}=\gamma\Vert \bm{D}^{(j)}_{t-1}\Vert_2^2$.
\STATE $\bm{C}_t^{(j)}=\Theta_{\lambda/\tau_{j,t-1}}(\hat{\bm{C}}_{t-1}^{(j)}-\bm{G}^{(j)}_{t-1}/\tau_{j,t-1})$.
\STATE $\hat{\bm{X}}=\hat{\bm{X}}+\bm{D}^{(j)}_{t-1}(\bm{C}_{t}^{(j)}-\hat{\bm{C}}_{t-1}^{(j)})$.
\ENDFOR
\STATE $\bm{\Delta}=\bm{C}_{t-1}-\bm{C}_{t}$.
\ENSURE$\bm{\Delta}$, $\bm{C}_t$.
\end{algorithmic}
\end{algorithm}

In Algorithm \ref{alg.GS}, we  update $\bm{C}^{(j)}$ sequentially for $j\in [k]$, which is not efficient when $k$ and $n$ are large. To improve the efficiency, we may use Jacobi method to update $\bm{C}$, which is shown in Algorithm \ref{alg.Jacobi} and can be implemented parallelly. 

\renewcommand{\algorithmicrequire}{\textbf{Input:}}
\renewcommand{\algorithmicensure}{\textbf{Output:}}
\begin{algorithm}[h]
\caption{Update $\bm{C}$ by Jacobi method}
\label{alg.Jacobi}
\begin{algorithmic}[1]
\REQUIRE
$\bm{X}$, $\bm{D}_{t-1}$, $\bm{C}_{t-1}$, $\gamma\geq1$.
\STATE $\bm{G}=-\bm{D}_{t-1}^\top(\bm{X}-\bm{D}_{t-1}\bm{C}_{t-1})$.
\STATE $\tau=\gamma\Vert \bm{D}_t \Vert_2^2$.
\FOR{$j=1,2,\ldots,k$}
\STATE$\bm{C}_t^{(j)}=\Theta_{\lambda/\tau}(\bm{C}_{t-1}^{(j)}-\bm{G}^{(j)}/\tau)$.
\ENDFOR
\ENSURE $\bm{C}_t$.
\end{algorithmic}
\end{algorithm}

\subsection{Update ${D}$}
After $\bm{C}^{(1)},\ldots,\bm{C}^{(k)}$ have been updated, we solve
\begin{equation}\label{eq.ksc_D}
\begin{aligned}
\mathop{\textup{minimize}}_{\bm{D}\in\mathbb{S}_D}&\ \ \dfrac{1}{2}\Vert \bm{X}-\bm{D}\bm{C}_t\Vert_F^2
\end{aligned}
\end{equation}
by projected gradient descent \cite{parikh2014proximal}. Specifically, for $u\in[\vartheta]$, 
\begin{equation}
\bm{D}_{t_u}=\mathcal{P}_{\Pi}\big(\bm{D}_{t_{u-1}}-\kappa_t^{-1}(\bm{X}-\bm{D}_{t_{u-1}}\bm{C}_t)(-\bm{C}_t^\top)\big),
\end{equation}
where $\kappa_{t}=\Vert \bm{C}_t\bm{C}_t^\top\Vert_2$ and $\mathcal{P}_{\Pi}$ denotes the column-wise projection onto unit ball defined by
\begin{equation}
\mathcal{P}_{\Pi}(\bm{v})=\left\{
\begin{array}{ll}
\bm{v}, &\textup{if}\ \Vert\bm{v}\Vert\leq 1;\\
\bm{v}/\Vert \bm{v}\Vert, &\textup{otherwise.}
\end{array}
\right.
\end{equation}
Algorithm \ref{alg.updateD} details the implementation. The following theorem provides the convergence rate of Algorithm \ref{alg.updateD}.
\begin{theorem}[Theorem 10.21 in \cite{beck2017first}]
Let $\bm{D}_t^\ast$ be the optimal solution of \eqref{eq.ksc_D}. Denote $\mathcal{L}(\bm{D}_{t_u})=\tfrac{1}{2}\Vert \bm{X}-\bm{D}_{t_u}\bm{C}_t\Vert_F^2$, where $\bm{D}_{t_u}\in\mathbb{S}_D$. Then in Algorithm \ref{alg.updateD},
\begin{equation}
\mathcal{L}(\bm{D}_{t_u})-\mathcal{L}(\bm{D}_t^\ast)\leq \dfrac{\kappa_t}{2u}\Vert \bm{D}_{t_u}-\bm{D}_t^\ast\Vert_F^2.
\end{equation}
\end{theorem}

\begin{algorithm}[h!]
\caption{Projected gradient method for $\bm{D}$}
\label{alg.updateD}
\begin{algorithmic}[1]
\REQUIRE
$\bm{X}$, $\bm{C}_t$, $\bm{D}_{t-1}$, $\vartheta$, 
\STATE $\bm{A}=\bm{X}\bm{C}_t^\top$, $\bm{B}=\bm{C}_t\bm{C}_t^\top$, and	$\kappa_t=\Vert \bm{B}\Vert_2$.
\STATE $\bm{D}_{t_0}=\bm{D}_{t-1}$.
	  \FOR{$u=1,2,\ldots,\vartheta$}
		\STATE $\bm{G}=-\bm{A}+\bm{D}_{t_{u-1}}\bm{B}$.
        \STATE $\bm{D}_{t_u}=\mathcal{P}_{\Pi}(\bm{D}_{t_{u-1}}-\bm{G}/\kappa_t)$.
      \ENDFOR
\ENSURE $\bm{D}_t=\bm{D}_{t_\vartheta}$.
\end{algorithmic}
\end{algorithm}

In fact there is no need to solve \eqref{eq.ksc_D} exactly because the problem about $\bm{C}$ \eqref{eq.ksc_C} is not exactly solved. We just set a small value (e.g. 5) for $\vartheta$ to obtain an inexact $\bm{D}_t$ and keep the time complexity low.

\subsection{The overall algorithm of k-FSC}
The entire algorithm of k-FSC is shown in Algorithm \ref{alg.k-FSC}, in which we have set default values for $T$, $\delta$, $\gamma$, $\vartheta$, and $\epsilon$ for convenience.
The space complexity is $O(mn+kmd+kdn)$ mainly caused by the storage for $\bm{X}$, $\bm{D}$, and $\bm{C}$. In the update of $\bm{C}$, the time complexity is $O(kdmn)$ mainly caused by line 9 and k loops of lines 11 and 14 in Algorithm \ref{alg.GS}. The time complexity of Algorithm \ref{alg.Jacobi} is lower than that in Algorithm \ref{alg.GS}. In the update of $\bm{D}$, the time complexity is $O(kdmn+\vartheta k^2d^2m)$ mainly contributed by line 1 and $\vartheta$ loops of line 4 in Algorithm \ref{alg.updateD}. In many real applications, $n\gg m>d$ holds. Thus by assuming $k^2d^2\leq n$, the time complexity in each iteration of Algorithm \ref{alg.k-FSC} is $O(kdmn+\vartheta mn)$. We see that the time complexity and space complexity of k-FSC are linear with the number of data points $n$. The time complexity of the k-means and line 9 in the initialization is much lower than that in computing $\bm{D}$ and $\bm{C}$.

\renewcommand{\algorithmicrequire}{\textbf{Input:}}
\renewcommand{\algorithmicensure}{\textbf{Output:}}
\begin{algorithm}[h!]
\caption{k-FSC}
\label{alg.k-FSC}
\begin{algorithmic}[1]
\REQUIRE
$\bm{X}$, $k$; $d$, $\lambda$; $T(200)$, $\delta(0.95)$, $\gamma(1)$, $\vartheta(5)$, $\epsilon(10^{-4})$; $t=0$, $\bm{\Delta}=\bm{0}$.
\STATE Normalize the columns of $\bm{X}$ to have unit $\ell_2$ norm
\STATE Generate $\bm{D}_0$ randomly or by $k$-means.
\STATE$\bm{C}_0=(\bm{D}_0^\top \bm{D}_0+\hat{\lambda}\bm{I})^{-1}\bm{D}_0^\top \bm{X}$.
\REPEAT
\STATE $t\leftarrow t+1$.
\STATE Obtain $\bm{C}_t$ using Algorithm \ref{alg.GS} or Algorithm \ref{alg.Jacobi}. 
\STATE Obtain $\bm{D}_t$ using Algorithm \ref{alg.updateD}.
\UNTIL{$\max\left(\tfrac{\Vert\bm{C}_t-\bm{C}_{t-1}\Vert_F}{\Vert\bm{C}_{t-1}\Vert_F},\tfrac{\Vert\bm{D}_t-\bm{D}_{t-1}\Vert_F}{\Vert\bm{D}_{t-1}\Vert_F}\right)\leq\epsilon$ or $t=T$}
\STATE Identify the clusters by \eqref{eq.cluster_assign_0} or \eqref{eq.cluster_assign2}.
\ENSURE $k$ clusters of $\bm{X}$.
\end{algorithmic}
\end{algorithm}

\begin{figure}[h!]
\centering
\includegraphics[width=5cm]{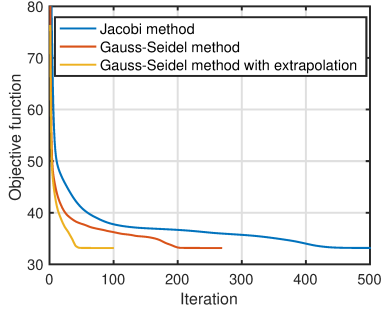}
\caption{Convergence performance of k-FSC (Algorithm \ref{alg.k-FSC}) with different solvers for $\bm{C}$ in clustering a synthetic dataset.}\label{fig_optcurve}
\end{figure}

Figure \ref{fig_optcurve} compares the convergence performance of k-FSC with different solvers for $\bm{C}$ on a synthetic dataset (see Section \ref{sec_exp_syn}). The  Gauss-Seidel method with extrapolation for $\bm{C}$ provides faster convergence than other methods, though Jacobi method can be implemented parallelly for large dataset with large $k$. The following theorem provides convergence guarantee for Algorithm \ref{alg.k-FSC}. 

\begin{theorem}\label{the_convergence}
Let $\mathcal{F}(\bm{C},\bm{D})= \dfrac{1}{2}\Vert \bm{X}-\bm{D}\bm{C}\Vert_F^2+\lambda\sum_{j=1}^k\Vert\bm{C}^{(j)}\Vert_{2,1}$. In Algorithm \ref{alg.k-FSC} (with Algorithm \ref{alg.GS} for $\bm{C}$), for any $\delta<1$, $\gamma\geq1$, and $\vartheta\geq 1$, we have
\begin{equation*}
\begin{aligned}
&\lim_{t\rightarrow\infty}\Vert \bm{C}_{t-1}-\bm{C}_{t}\Vert_F=0,\quad \lim_{t\rightarrow\infty}\Vert \bm{D}_{t-1}-\bm{D}_{t}\Vert_F=0,\\
&\lim_{t\rightarrow \infty} \mathcal{F}(\bm{C}_{t-1},\bm{D}_{t-1})-\mathcal{F}(\bm{C}_t,\bm{D}_t)= 0.
\end{aligned}
\end{equation*}
\end{theorem}

\section{Extensions of  k-FSC}\label{sec_4}


\subsection{Online and mini-batch optimizations}
In many real cases, we observe data points sequentially, which can be well exploited by online learning. In addition, online learning often has low memory cost and low per-iteration time complexity and hence is able to handle very large datasets. K-FSC can be extended to online clustering or solved by mini-batch optimization. Specifically, given a mini-batch data $\bm{X}_i\in\mathbb{R}^{m\times b}$ at time $i$, we update $\bm{D}$ by inexactly solving the following problem
\begin{equation}\label{eq.ksc_online}
\begin{aligned}
\mathop{\textup{minimize}}_{\bm{D}\in\mathbb{S}_D,\bm{C}_i}&\ \ \dfrac{1}{2}\Vert \bm{X}_i-\bm{D}\bm{C}_i\Vert_F^2+\lambda\sum_{j=1}^k\Vert\bm{C}_i^{(j)}\Vert_{2,1}.
\end{aligned}
\end{equation}
When $b=1$, the problem is exactly an online optimization problem.
The complete algorithm is shown in Algorithm \ref{alg.k-FSC_online}.
\begin{algorithm}[h!]
\caption{Mini-Batch k-FSC (k-FSC-MB)}
\label{alg.k-FSC_online}
\begin{algorithmic}[1]
\REQUIRE
Sequential data or randomly partitioned data $\bm{X}_1,\bm{X}_2,\ldots,\bm{X}_i\in\mathbb{R}^{m\times b}$, $k$, $d$, $\lambda$, $\lambda'$, $T_C$, $T_D$.
\STATE Initialize $\bm{D}$ randomly or use k-means.
	  \FOR{$i=1,2,\ldots$}
	  \STATE  Initialize $\bm{C}_i$
	  \STATE  Update $\bm{C}_i$ by performing Algorithm \ref{alg.GS} for $T_C$ times.
		\STATE  Update $\bm{D}$ by performing Algorithm \ref{alg.updateD} with $\vartheta=T_D$.
		\STATE  Cluster (similar to lines 8 and 9) if necessary.
      \ENDFOR
\STATE Repeat lines $2,4,5,6$ for $p$ times if necessary.
\STATE  $\bm{C}^{(j)}=({\bm{D}^{(j)}}^\top\bm{D}^{(j)}+\lambda'\bm{I})^{-1}{\bm{D}^{(j)}}^\top\bm{X}$,\ $j\in [k]$.
\STATE $\bm{x}_i\in c_j,\quad  j=\textup{argmin}_j\ \Vert \bm{x}_i-\bm{D}^{(j)}\bm{C}_{:i}^{(j)}\Vert^2,\ i\in[n]$.
\ENSURE $k$ clusters of $\bm{X}$.
\end{algorithmic}
\end{algorithm}

\vspace{-4pt}
\subsection{Cluster arbitrarily large dataset}
Though the time and space complexity of k-FSC are linear with $n$, an extremely large $n$ (e.g. $n\geq10^6$) will still lead to high computational cost. In that case, we propose to perform k-FSC on a few landmark data points generated by k-means and then perform classification. The method is detailed in Algorithm \ref{alg.k-FSC_large}. The time complexity per iteration of line 2 (i.e. Algorithm \ref{alg.k-FSC}) is $O(kdms+\vartheta k^2d^2m)$. The time complexity of line 3 and line 4 is $O(kdmn)$. 

\begin{algorithm}[h!]
\caption{k-FSC for arbitrarily large dataset (k-FSC-L)}
\label{alg.k-FSC_large}
\begin{algorithmic}[1]
\REQUIRE
$\bm{X}\in\mathbb{R}^{m\times n}$, $s\ll n$, $\lambda'$ (e.g. $10^{-5}$).
\STATE Let $\bm{X}_s$ consist of the $s$ centers of k-means on $\bm{X}$.
\STATE Run Algorithm \ref{alg.k-FSC} on  $\bm{X}_s$ to obtain $\bm{D}$.
\STATE $\bm{C}^{(j)}=({\bm{D}^{(j)}}^\top\bm{D}^{(j)}+\lambda'\bm{I})^{-1}{\bm{D}^{(j)}}^\top\bm{X}$,\ $j\in [k]$.
\STATE $\bm{x}_i\in c_j,\quad  j=\textup{argmin}_j\ \Vert \bm{x}_i-\bm{D}^{(j)}\bm{C}_{:i}^{(j)}\Vert^2,\ i\in[n]$.
\ENSURE $k$ clusters of $\bm{X}$.
\end{algorithmic}
\end{algorithm}

\vspace{-4pt}
\subsection{Complexity comparison}
We analyze the time and space complexity of a few baseline methods. Shown in Table \ref{tab_complexity}, the space complexity of Nystr{\"o}m-ort \cite{fowlkes2004spectral}, k-PC \cite{agarwal2004k}, k-FSC, and k-FSC-L are much lower than other methods when $n\gg m$. The space complexity of LSC-K \cite{chen2011large} and RPCM-F$^2$ \cite{9292470} increase quickly when $s$ becomes larger. 

For extremely large dataset, in order to achieve high clustering accuracy, we often need a large enough $s$ to exploit sufficient information of the dataset. In k-FSC-L,  the complexity is linear with $s$, which means we may obtain high clustering accuracy by k-FSC-L on extremely large datasets. In contrast, the time complexity of Nystr{\"o}m-ort, LSC-K \cite{chen2011large}, SSSC \cite{peng2013scalable}, RPCM-F$^2$ \cite{9292470}, S$^5$C \cite{matsushima2019selective}, and S$^3$COMP-C \cite{chen2020stochastic}, are at least quadratic with $s$, which prevents their applications in large-scale clustering demanding high accuracy.


\begin{table}[h!]
\centering
\caption{Time and space complexity ($\rho<1$: proportion of nonzero entries; $k$: number of clusters; $s$: number of selected samples; $b>1$, $\epsilon<1$; $\delta<1$, e.g. $0.8$; $\vartheta\geq 1$, e.g. $5$).}\label{tab_complexity}
\begin{tabular}{l|c|c|c}
\hline
& \multirow{2}{*}{{\small Space complexity}} & \multicolumn{2}{c}{Time complexity} \\  \cline{3-4}
& & Iterative & Fixed \\ \hline
k-PC \cite{agarwal2004k} & $O(mn$+$kmd)$ 	& \makecell{$O(dmn$+\\$kdm^2$+$m^2n)$} & ---\\
SSC \cite{SSC_PAMIN_2013} &$O(mn$+$\rho n^2)$ 	& $O(mn^2)$& $O(k\rho n^2)$\\
{\small Nystr{\"o}m \cite{fowlkes2004spectral}} & $O(mn)$	& --- &$O(msn$+$s^3)$\\
LSC-K \cite{chen2011large}& $O(mn$+$sn)$	&  --- &$O(msn$+$s^2n)$\\
SSSC \cite{peng2013scalable} & $O(mn$+$\rho s^2)$		& $O(ms^3$+$k^2s)$ & $O(s^2n)$\\
{\footnotesize RPCM-F$^2$ \cite{9292470}} & $O(mn$+$s^2$+$sn)$ 	& $O(ms^2)$ & $O(msn$+$s^2n)$ \\
S$^5$C \cite{matsushima2019selective}& $O(mn$+$\rho n^2)$ 	& --- & \makecell{$O(bms^2$+$msn$\\+$ksn\log\tfrac{1}{\epsilon})$}\\
{\footnotesize S$^3$COMP-C\cite{chen2020stochastic}} & $O(mn$+$\rho n^2)$ & $O(m\rho n^3(1-\delta))$ &$O(k\rho n^2)$\\ \hline
k-FSC &\makecell{$O(mn$+$kmd$\\+$kdn)$} & \makecell{$O(kdmn$\\+$\vartheta mn)$}&$O(kdmn)$\\
k-FSC-MB &\makecell{$O(mb$+$kmd$\\+$kdb)$ }&\makecell{$O(kdmb$\\+$\vartheta k^2d^2m)$}&$O(kdmn)$\\
k-FSC-L &\makecell{$O(ms$+$kmd$\\+$kds)$} &\makecell{ $O(kdms$\\+$\vartheta k^2d^2m)$}&$O(kdmn)$\\ \hline
\end{tabular}
\end{table}

\subsection{Sparse noise, outliers, and missing data}
In real applications, sparse noise, outliers, and missing data are not uncommon \cite{FAN201736,fan2019factor}. With slight modification from model \eqref{eq.ksc_2}, k-FSC is able to handle sparse noise, outliers, or/and missing data. 
For instance, the following model is robust to sparse noise or outliers
\begin{equation}\label{eq.ksc_outliers}
\begin{aligned}
\mathop{\textup{minimize}}_{\bm{D}\in\mathbb{S}_D,\bm{C}}&\ \ \dfrac{1}{2}\Vert \bm{X}-\bm{D}\bm{C}-\bm{E}\Vert_F^2+\lambda\sum_{j=1}^k\Vert\bm{C}^{(j)}\Vert_{2,1}+\beta \mathcal{R}(\bm{E}),
\end{aligned}
\end{equation}
where $\mathcal{R}(\bm{E})=\Vert \bm{E}\Vert_1$ or $\Vert \bm{E}\Vert_{2,1}$.
The following model is able to perform clustering and missing data imputation simultaneously.
\begin{equation}\label{eq.ksc_miss}
\begin{aligned}
\mathop{\textup{minimize}}_{\bm{D}\in\mathbb{S}_D,\bm{C}}&\ \ \dfrac{1}{2}\Vert \bm{M}\odot(\bm{X}-\bm{D}\bm{C})\Vert_F^2+\lambda\sum_{j=1}^k\Vert\bm{C}^{(j)}\Vert_{2,1},
\end{aligned}
\end{equation}
where $\odot$ denotes the Hadamard product and $\bm{M}$ is a binary matrix with $1$ for observed entries and $0$ for missing entries. The missing entries can be obtained from $\bm{D}\bm{C}$. The optimizations for \eqref{eq.ksc_outliers} and \eqref{eq.ksc_miss} can be adapted from Algorithm \ref{alg.k-FSC} and will not be detailed here.

\section{Connection with previous work}\label{sec_5}
The proposed k-FSC has a connection with nonnegative matrix factorization (NMF) \cite{lee2001algorithms}.
It is known that k-means clustering can be formulated as NMF \cite{ding2005equivalence}. Therefore, both NMF and k-FSC factorize the data into k clusters directly. The difference is that NMF aims to find the cluster centers while k-FSC aims to find the subspaces. 

K-FSC is also closely related to the k-plane clustering (k-PC) \cite{bradley2000k,he2016robust}, 
which aims to minimize the sum of residuals of data points to their assigned subspace.
An efficient method to solve  k-PC is performing assignment and learn the subspace bases alternately: cluster the data points by their nearest subspaces and update the subspace bases by PCA on the data points in each cluster. K-PC is sensitive to initialization \cite{he2016robust},  subspace dimension estimation, missing data, and outliers \cite{gitlin2018improving}.

The model of k-FSC can be regarded as a variant of dictionary learning and sparse coding (DLSC) \cite{mairal2009online}. Similar to \cite{szabo2011online,5539964,suo2014group}\footnote{Structured dictionary was also considered in compressed sensing \cite{eldar2010block} but the dictionary is not unknown in that case.}, k-FSC also considers structured dictionary. It is worth pointing out that, in these previous work, the regularization on the coefficients matrix is $\ell_1$ norm. In contrast, k-FSC puts $\ell_{21}$ norm on the $k$ sub-matrices of the coefficients matrix to make it be group-sparse, which enables us to factorize the data matrix into $k$ groups directly. In \cite{sprechmann2010dictionary}, the authors proposed to perform DLSC and clustering alternately, which is-time consuming and not applicable to large datasets.

\section{Experiments}\label{sec_exp}

\subsection{Synthetic data}\label{sec_exp_syn}
This paper generates\footnote{All experiments in this paper are conducted in MATLAB on a MacBook Pro with 2.3 GHz Intel Core i5 and 8 GB RAM.} synthetic data $\bm{X}=[\bm{X}^{(1)}, \ldots, \bm{X}^{(k)}]$ by
$\bm{X}^{(j)}=(\alpha\bm{A}_0+\bm{A}^{(j)})\bm{B}^{(j)}$. Here
$\bm{A}^{(j)}\in\mathbb{R}^{m\times d_0}$ and $\bm{B}^{(j)}\in\mathbb{R}^{d_0\times n_0}$ are drawn from $\mathcal{N}(0,1)$, $j\in [k]$. $\bm{A}_0$ is a random matrix drawn from $\mathcal{N}(0,1)$ and $\alpha$ controls the similarity between pair-wise subspaces. We also add random noise to $\bm{X}$: $\hat{\bm{X}}=\bm{X}+\bm{E}$, where $\bm{E}$ is drawn from $\mathcal{N}(0,(\beta\sigma_x)^2)$, $\sigma_x$ denotes the standard deviation of the entries in $\bm{X}$, and $\beta$ controls the noise level.  We set $k=5$, $m=25$, $d_0=5$, $n_0=50$, and $\alpha=1$.

\begin{figure}[h!]
\centering
\includegraphics[width=8.5cm]{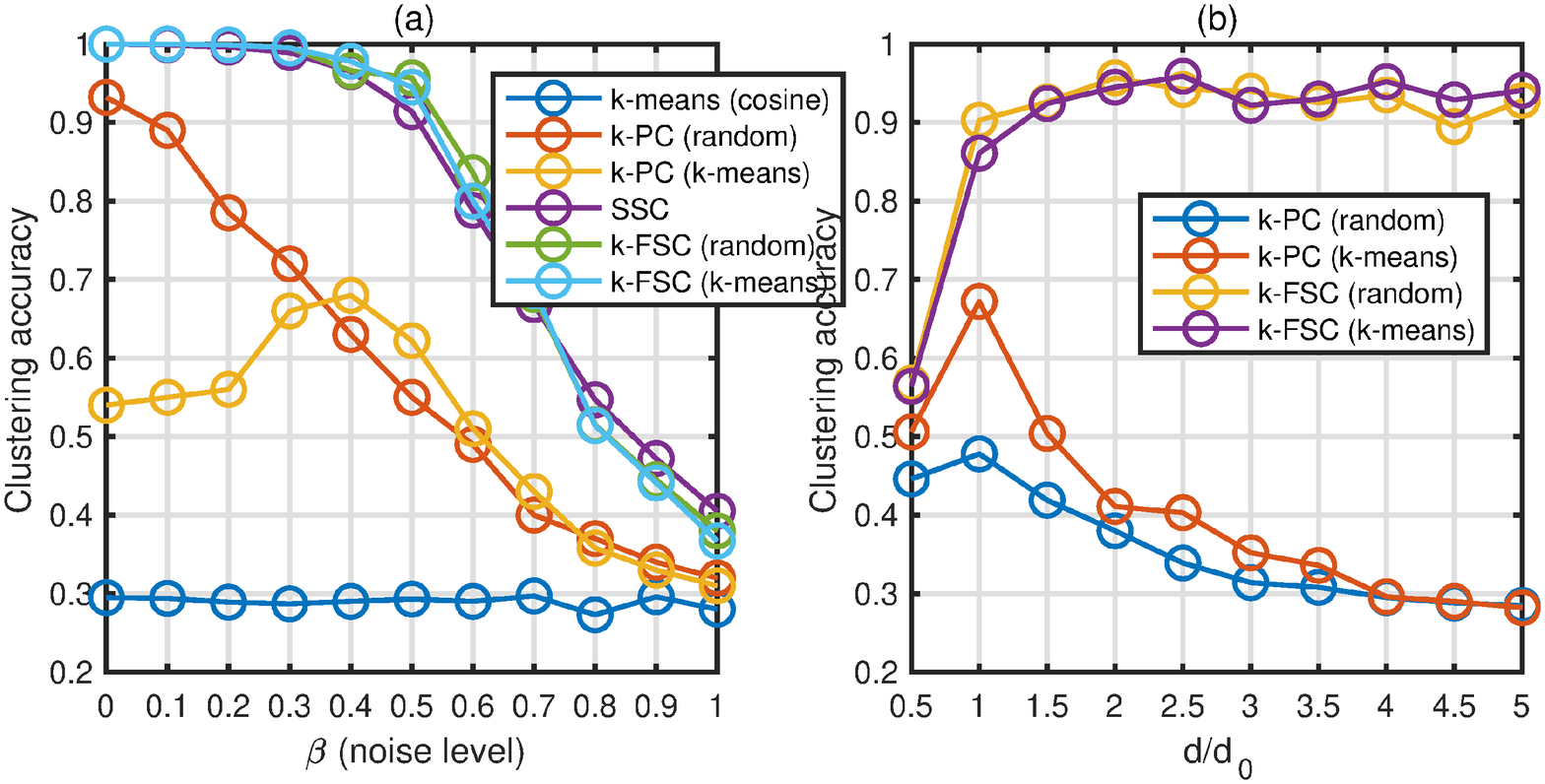}
\caption{Clustering accuracies of k-means, k-PC, SSC, and k-FSC (with random or k-means initialization): (a) different noise level; (b) k-PC and k-FSC with different estimation ($d$) of subspace dimension, where $\beta=0.5$.}\label{fig_syn_noise_d}
\end{figure}

Figure \ref{fig_syn_noise_d}(a) shows the clustering accuracy (average of 50 trials) of k-means (cosine distance), k-PC  \cite{agarwal2004k} with $d=d_0$, SSC \cite{SC_tutorial2011}, and the proposed k-FSC ($d=2d_0$) in the cases of different noise level. Random initialization and k-means (cosine distance) initialization for k-PC and k-FSC are also compared. We see that k-means failed in all cases though we have used cosine as a distance metric. K-FSC outperformed SSC when the noise level was relatively large; they outperformed k-PC in all cases. Note that in this study, as k-means failed, initialization by k-means provided no significant improvement compared to random initialization.
Figure \ref{fig_syn_noise_d} (b) presents the influence of $d$ in k-PC and k-FSC when $\beta=0.5$.  We see that k-PC requires $d$ be equal to the true dimension $d_0$, otherwise the clustering accuracy decreases quickly when $d$ increases. In contrast, k-FSC is not sensitive to $d$, even when $d$ is five times of the true dimension of the subspaces. In addition, k-FSC is also not sensitive to $\lambda$, which can be found in Appendix \ref{app_DL}.

To test the clustering performance of k-FSC when the data are corrupted by sparse noise, we use $\hat{\bm{X}}=[\bm{X}^{(1)}, \ldots, \bm{X}^{(k)}]+\bm{E}+\bm{F}$, where $\bm{E}$ was defined previously and $\bm{F}$ is a sparse matrix whose nonzero entries are drawn from $\mathcal{N}(0,\sigma_x^2)$. We let $\beta=0.1$ and increase the proportion of the nonzero entries (noise density of sparse noise) of $\bm{F}$ from $0$ to $0.5$. The clustering accuracy of k-PC, SSC, and k-FSC are reported in Figure \ref{Fig_syn_SM}(a). We see that k-PC is very vulnerable to the sparse noise. Compared to SSC, k-FSC is more robust to the sparse noise and the clustering accuracy is always higher than 0.9 when the noise density is no larger than 0.4.

\begin{figure}[h!]
\centering
\includegraphics[width=8.5cm,trim={30 0 40 0},clip]{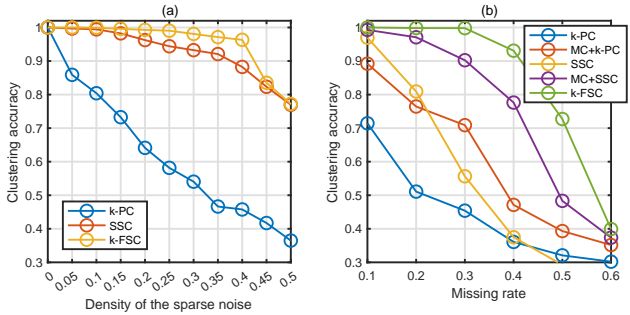}
\caption{(a) sparse noise; (b) missing data.}\label{Fig_syn_SM}
\end{figure}

We randomly remove a fraction (missing rate) of the entries of the data matrix ($\beta=0.1$) and test the performance of k-PC, SSC, and k-FSC. In k-PC and SSC, we fill the missing entries with zero. We also use low-rank matrix completion \cite{CandesRecht2009} to recover the matrix and then perform k-PC and SSC. The clustering accuracy of 50 repeated trials are reported in Figure \ref{Fig_syn_SM}(b). We see matrix completion has improved the clustering accuracy of k-PC and SSC. Nevertheless,  k-FSC has the highest clustering accuracy in all cases. It is worth mentioning that the data matrix is full-rank and hence cannot be well recovered by low-rank matrix completion. That's why the proposed method outperformed MC+k-PC and MC+SSC.

\subsection{Real data}
We compare k-FSC (Algorithm \ref{alg.k-FSC}), k-FSC-MB (Algorithm \ref{alg.k-FSC_online}), and k-FSC-L (Algorithm \ref{alg.k-FSC_large}) with k-means (cosine similarity), k-PC \cite{agarwal2004k}, SSC  \cite{SSC_PAMIN_2013}, Nystr\"om-orth \cite{fowlkes2004spectral}, LSC-K \cite{chen2011large}, SSSC \cite{peng2013scalable}, RPCM-F$^2$ \cite{9292470}, S$^5$C \cite{matsushima2019selective}, and S$^3$COMP-C \cite{chen2020stochastic}. We use the MATLAB codes shared by their authors. The evaluation are conducted on the following six datasets.
\textbf{MNIST}: \cite{lecun1998gradient} 70,000 gray images ($28\times 28$) of handwritten digits. Similar to \cite{chen2020stochastic}, for each image, we use the scattering convolution network \cite{bruna2013invariant} to generate a feature vector of dimension $3472$ further reduced to 150 by PCA (use the first 150 right singular vectors of the matrix).
\textbf{Fashion-MNIST}: \cite{xiao2017fmnist} 70,000 gray images ($28\times 28$) of 10 types of fashion product. The preprocessing is the same as that for MNIST.
\textbf{Epileptic}: \cite{andrzejak2001indications} EEG data with 178 features  and 11,500 samples in 5 classes. We reduced the feature dimension to 50 by PCA (use the first 50 right singular vectors of the matrix).
\textbf{Motion Capture Hand Postures}: A UCI \cite{Dua:2019} dataset with 38 features and 78,095 samples in 5 classes.
\textbf{Covtype}:  A UCI \cite{Dua:2019} dataset with 54 features and 581,012 samples in 7 classes.
\textbf{PokerHand}:  A UCI \cite{Dua:2019} dataset with 10 features and 1,000,000 samples in 10 classes. All data are normalized to have unit $\ell_2$ norm. 
\begin{figure}[h]
\centering
\includegraphics[width=8cm,trim={100 20 100 5},clip]{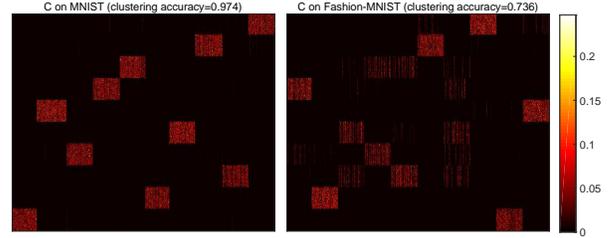}
\caption{Visualization of $\vert\bm{C}\vert$ given by k-FSC.}\label{fig_real_C}
\end{figure}
\begin{table}[h!]
\centering
\caption{Performance on MNIST and Fashion-MNIST}\label{tab_MNIST}
\begin{tabular}{c|l|c|c|c}
\hline 
&	& ACC ($\%$) & NMI ($\%$) & Time (s) \\ \hline \hline
\multirow{11}{*}{\rotatebox{90}{MNIST}} &k-means & 95.72$\pm$3.38 & 91.46$\pm$3.84 & 28.7 \\ 
&k-PC & 87.67$\pm$6.75 & 83.21$\pm$4.75 & 41.9 \\ 
&Nystr\"om & 78.56$\pm$7.13 & 76.49$\pm$3.80 & 60.3\\ 
&LSC-K & 95.83$\pm$1.09 & 90.91$\pm$0.78 & 296.2\\ 
&SSSC &82.93$\pm$0.39 & 83.44$\pm$0.53 & 127.7 \\ 
&NCSC & \underline{94.09} &\underline{86.12} & {\scriptsize Need GPU} \\
&RPCM-F$^2$ & 96.95$\pm$0.19 &91.87$\pm$0.31 & 54.5 \\ 
&S$^5$C & 94.86$\pm$1.37 & 89.85$\pm$1.13 & 291.6\\ 
&{\small S$^3$COMP-C}& \underline{96.32} &\slash & \slash \\ \cline{2-5}
&k-FSC &\textbf{97.24}$\pm$0.02 &\textbf{92.58}$\pm$0.05  & 335.9 \\ 
&k-FSC-MB& 97.13$\pm$0.04 & 92.30$\pm$0.09 & 55.2\\
&k-FSC-L & \textbf{97.48}$\pm$0.31 & \textbf{93.45}$\pm$0.45 & 36.9 \\ 
\hline \hline
\multirow{11}{*}{\rotatebox{90}{Fashion-MNIST}}&k-means & 65.51$\pm$5.05 & 65.23$\pm$3.03 & 33.4 \\ 
&k-PC & 61.88$\pm$6.31 & 60.78$\pm$4.01 & 48.8\\ 
&Nystr\"om & 54.62$\pm$3.67 & 48.33$\pm$1.40 & 60.2\\ 
&LSC-K & 63.27$\pm$2.77 & 65.60$\pm$1.65 & 290.3\\ 
&SSSC &57.90$\pm$1.48 & 60.69$\pm$0.88 & 121.4 \\ 
&NCSC & \underline{72.14} &\underline{68.60} & {\scriptsize Need GPU} \\
&RPCM-F$^2$ & 65.98$\pm$3.19 &67.23$\pm$1.95 & 55.8 \\ 
&S$^5$C & 63.13$\pm$1.63 & 66.38$\pm$1.34 & 297.2\\ 
&{\small S$^3$COMP-C}& 59.88$\pm$2.19 &65.00$\pm$0.17 & $\widetilde{762.6}$\\ \cline{2-5}
&k-FSC &\textbf{72.73}$\pm$3.13 &\textbf{69.24}$\pm$1.94 & 527.5 \\
&k-FSC-MB & \textbf{71.51}$\pm$4.08 & 68.08$\pm$3.09 & 58.5 \\ 
&k-FSC-L & 69.70$\pm$4.32 & \textbf{68.23}$\pm$2.45 & 57.7 \\ 
\hline
\end{tabular}
\end{table}

The following parameter settings are used for the six datasets. 
In k-PC, we set $d=6,10,6,2,5,3$. In Nystr\"om-orth, $\sigma=$0.25,
0.5,0.3,0.5,0.2,1 and $s=3000$,3000,2000,1500,1500,1000. In LSC-K, $r=5,4,5,5,5,5,$ and $s=3000$, 3000,2000,1500,1000,1000. In SSSC, $\lambda=10^{-4}$,$10^{-2}$, $10^{-1}$,$10^{-1}$, $10^{-1}$ and $s=$3000, 3000,
1500,1500,3500. In RPCM-F$^2$, $\beta=0.1,0.1,10,0.5,0.1$ and $s=3000$,3000,1500,1500,1000. In S$^5$C, $\lambda=0.2,0.2,0.3,0.2$ and $s=3000,3000,1500,1500$. In S$^3$COMP-C, $T=10$, $\beta=0.8$, and $\lambda=1,1,0.4,0.5$ for the first four datasets. 

In k-FSC, k-FSC-MB, and k-FSC-L, for the six datasets, we set $d=30$,30,30,30,20,5, $\lambda=0.5$,0.5,0.5,0.2,0.4,0.1, and $\lambda'=10^{-5}$.
In k-FSC-MB, we set $b=1000$ and $p=5,5,20,10,2,2$. In k-FSC-L, we set $s=5000,5000,1500,2500,3500,5000$, namely $s=500k$ (except Epileptic because it is a relatively small dataset).
The numbers of repetitions of k-means in k-PC, LSC-K, k-FSC, and k-FSC-MB are 10 on all datasets. The number of repetitions of k-means in k-FSC-L is 100 on all datasets. Note that according to Theorem \ref{the_d}, $d$ can be arbitrarily large. But in practice we just use a relatively small $d$ (according to the data dimension $m$) to reduce the computational cost. Since the initialization of $\bm{D}$ may not be good enough, we still need to tune $\lambda$ under the guidance of Theorem \ref{the_lambda}.

Figure \ref{fig_real_C} shows two examples of $\bm{C}$ given by k-FSC on MNIST and Fashion-MNIST. We see that k-FSC can find the cluster blocks effectively.
The average clustering accuracy (ACC), normalized mutual information (NMI), and time cost\footnote{The time cost is the total cost of all procedures. The underlined values are the results reported in the original papers. The `\slash'  means out-of-memory or exceeding 3 hours. On Fashion-MNIST and Postures, S$^3$COMP-C is out of memory. So we perform S$^3$COMP-C on two subsets ($20\%$) of Fashion-MNIST and Postures. The time costs of S$^3$COMP-C can be reduced if performed in parallel.} of ten repeated trials on MNIST and Fashion-MNIST are reported in Table \ref{tab_MNIST}, in which we also compare NCSC \cite{zhang2019neural} (a deep learning method). We see that k-FSC, k-SFC-MB, and k-FSC-L outperformed other methods in terms of ACC and NMI.  Meanwhile,  k-FSC-MB and k-FSC-L are more efficient than most methods such as S$^5$C  and S$^3$COMP-C.

\begin{table}[h!]
\centering
\caption{Performance on Epileptic and Postures}\label{tab_Epileptic}
\begin{tabular}{c|l|c|c|c}
\hline 
&	& ACC ($\%$) & NMI ($\%$) & Time (s) \\ \hline \hline
\multirow{9}{*}{\rotatebox{90}{Epileptic}}&k-means & 23.88$\pm$0.09 & 0.84$\pm$0.02 & 1.3 \\ 
&k-PC & 42.03$\pm$1.76& 18.12$\pm$0.91 & 2.1\\ 
&Nystr\"om & 27.21$\pm$2.84 & 4.91$\pm$1.55 & 11.2\\ 
&LSC-K & 33.75$\pm$0.27 & 14.96$\pm$0.23 & 9.8\\ 
&SSSC &38.14$\pm$3.27 & 19.41$\pm$2.76 & 27.1 \\ 
&RPCM-F$^2$ & 38.01$\pm$2.43 &16.42$\pm$1.07 & 2.3 \\ 
&S$^5$C & 41.42$\pm$2.15 & 22.08$\pm$1.79 & 24.3\\  
&{\small S$^3$COMP-C}& 41.39$\pm$3.68  &\textbf{26.04}$\pm$2.38 & $\widetilde{436.5}$ \\ \cline{2-5}
&k-FSC &43.26$\pm$2.16  &23.82$\pm$1.12 & 21.7 \\ 
&k-FSC-MB & \textbf{43.49}$\pm$1.75 & 24.01$\pm$0.98 & 9.1 \\
&k-FSC-L & \textbf{45.40}$\pm$0.98 & \textbf{24.29}$\pm$1.33  & 5.9 \\ 
\hline \hline
\multirow{9}{*}{\rotatebox{90}{Postures}}&k-means & 42.68$\pm$2.12 & 33.61$\pm$0.87 & 7.2\\ 
&k-PC & 41.41$\pm$3.41& 21.33$\pm$3.65 & 9.2\\ 
&Nystr\"om & 43.27$\pm$2.78 & 32.35$\pm$0.82 & 21.6\\ 
&LSC-K & 46.40$\pm$2.44 & 37.24$\pm$1.68 & 207.7\\ 
&SSSC &45.39$\pm$3.24 & 36.71$\pm$1.02 & 20.1\\ 
&RPCM-F$^2$ & 47.02$\pm$2.71 &36.41$\pm$2.15 & 23.0 \\ 
&S$^5$C &46.67$\pm$0.41 & \textbf{38.66}$\pm$1.48 & 451.8\\   
&{\small S$^3$COMP-C}& 45.26$\pm$3.38  &36.24$\pm$1.49 & $\widetilde{755.3}$ \\ \cline{2-5}
&k-FSC & \textbf{51.65}$\pm$2.26 & \textbf{39.39}$\pm$0.74& 173.9 \\ 
&k-FSC-MB & 49.97$\pm$2.29 & 36.15$\pm$1.73& 24.6\\
&k-FSC-L & \textbf{51.10}$\pm$4.73& 38.18$\pm$2.17& 9.8 \\ 
\hline
\end{tabular}
\end{table}

\begin{table}[h!]
\centering
\caption{Performance on Covtype and PokerHand}\label{tab_Covtype}
\begin{tabular}{c|l|c|c|c}
\hline 
&	& ACC ($\%$) & NMI ($\%$) & Time (s) \\ \hline \hline
\multirow{9}{*}{\rotatebox{90}{Covtype}}&k-means & 20.84$\pm$0.00 & 3.69$\pm$0.00 & 156.6 \\ 
&k-PC & 37.45$\pm$4.16& 5.09$\pm$0.51 & 123.7\\ 
&Nystr\"om & 23.18$\pm$0.90 & 3.75$\pm$0.01 & 635.8\\ 
&LSC-K & 24.16$\pm$1.29 & 5.73$\pm$0.08 & 4792.5\\ 
&SSSC &30.02$\pm$1.46 & 6.48$\pm$0.31 & 332.6\\ 
&RPCM-F$^2$ & 23.66$\pm$0.53 &3.75$\pm$0.11 & 2362.2 \\ 
&S$^5$C & \slash & \slash & \slash\\    
&{\small S$^3$COMP-C}& \slash  & \slash & \slash \\ \cline{2-5}
&k-FSC & \textbf{43.95}$\pm$3.46 & 5.59$\pm$1.64& 1762.6 \\ 
&k-FSC-MB & 41.31$\pm$3.27 & \textbf{7.70}$\pm$3.76 & 60.4\\
&k-FSC-L & \textbf{43.72}$\pm$2.95 & \textbf{6.92}$\pm$2.77& 19.6 \\ 
\hline \hline
\multirow{8}{*}{\rotatebox{90}{PokerHand}}&k-means & 10.47$\pm$0.05 & 0.04$\pm$0.00 & 169.3 \\ 
&k-PC & 12.43$\pm$0.42& 0.17$\pm$0.05 & 306.5\\ 
&Nystr\"om & 10.91$\pm$0.15 & 0.08$\pm$0.03 & 995.6\\ 
&LSC-K & \underline{12.32} & \underline{0.00} & \underline{8829.0}\\ 
&SSSC &\underline{19.31} & \underline{0.20} & \underline{474.1}\\ 
&RPCM-F$^2$ & \slash &\slash & \slash \\ 
&S$^5$C & \slash & \slash & \slash \\  
&{\small S$^3$COMP-C}& \slash  & \slash & \slash \\ \cline{2-5}
&k-FSC & {21.82}$\pm$2.18 & \textbf{0.33}$\pm$0.13& 1017.8\\ 
&k-FSC-MB & \textbf{33.15}$\pm$7.09 & 0.21$\pm$0.14 & 33.2 \\
&k-FSC-L & \textbf{22.19}$\pm$3.13 & \textbf{0.39}$\pm$0.15 & 18.6 \\ 
\hline
\end{tabular}
\end{table}

The results on Epileptic and postures are shown in Table \ref{tab_Epileptic}. In terms of ACC, the proposed methods outperformed all other methods. In terms of NMI, the proposed methods outperformed all other methods except S$^5$C and S$^3$COMP-C that are time-consuming..

The results on Covtype and PokerHand are reported in Table  \ref{tab_Covtype}. These two datasets are more challenging than the previous four datasets because the clusters are highly imbalanced, which will lead to low NMI.  On PokerHand, the results of LSC-K and SSSC are from \cite{peng2013scalable}. Since the datasets are too large, S$^5$C and S$^3$COMP-C do not apply. The ACCs of the proposed methods are much higher than other methods. Moreover, the time costs of k-FSC-MB and k-FSC-L are much lower than other methods.

\section{Conclusion}\label{sec_con}
This paper has presented a linear-complexity method k-FSC for subspace clustering. K-FSC is able to handle arbitrarily large dataset, streaming data, sparse noise, outliers, and missing data. Extensive experiments showed that k-FSC and its extensions are more accurate and efficient than state-of-the-art methods of subspace clustering. This improvement stems from the following aspects. First, k-FSC, k-FSC-MB, and k-FSC-L can utilize much more data points in the learning step while most of the other methods require the subset be small enough to ensure the scalability. Second, in the proposed methods, the number of clusters, as an important information, is directly exploited. Other methods except k-PC do not use the information before the spectral clustering step. K-FSC-MB and k-FSC-L are very efficient in handling very large datasets and are as accurate as k-FSC is. 
Future study may focus on the sufficient conditions for k-FSC to succeed.

\section*{Acknowledgements}
The work was supported by the research funding T00120210002 of Shenzhen Research Institute of Big Data. The author appreciates the reviewers' valuable time and comments.



\bibliography{Ref_SC}
 \bibliographystyle{ACM-Reference-Format}

\appendix
\section{More results on synthetic data}\label{app_DL}
In Figure \ref{Fig_syn_DL}(a), k-PC does not work when $d/d_0>5$ because the subspace dimension is equal or larger than the data dimension($m=25, d_0=5$). K-FSC always has high clustering accuracy even when $d=50d_0$. Figure \ref{Fig_syn_DL}(b) shows the clustering accuracy of k-FSC ($d=2d_0$) with different hyper-parameter $\lambda$ in the cases of different noise level. We see that k-FSC works well with a large range of $\lambda$ especially when the noise level is low. 
\begin{figure}[h!]
\centering
\includegraphics[width=6.8cm,trim={40 0 40 7},clip]{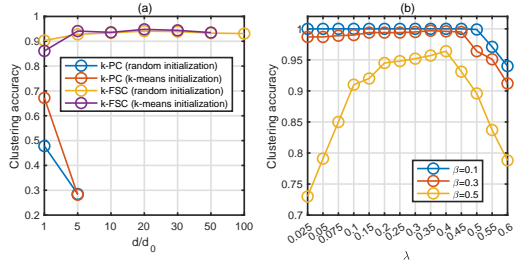}
\caption{(a) sensitivity of k-PC/k-FSC to $d$ ($\beta=0.5$); (b) sensitivity of k-FSC to  $\lambda$ in the cases of different noise level $\beta$.}\label{Fig_syn_DL}
\end{figure}



\section{Proof for Proposition \ref{proposition_1}}
\begin{proof}
(a)  To prove, we need to show that: \textcircled{1}  $\sum_{j=1}^k\Vert{\bm{C}}^{(j)}\Vert_{2,0}$ has a minimum under the constraint; \textcircled{2} when the minimum of $\sum_{j=1}^k\Vert{\bm{C}}^{(j)}\Vert_{2,0}$ is attained, all columns of $\bm{X}$ are correctly clustered.

For \textcircled{1}. Obviously, when all columns of $\bm{X}$ are correctly clustered according to \eqref{eq.cluster_assign_0}, we have $\sum_{j=1}^k\Vert\bm{C}^{(j)}\Vert_{2,0}=n$. If $\sum_{j=1}^k\Vert\bm{C}^{(j)}\Vert_{2,0}<n$, $\bm{C}$ has at least one zero column, which means the corresponding column of $\bm{X}$ can not be reconstructed and $\bm{X}\neq\bm{D}\bm{C}$. Therefore, under the constraint $\bm{X}=\bm{D}\bm{C}$, we have $\sum_{j=1}^k\Vert{\bm{C}}^{(j)}\Vert_{2,0}\geq n$ and the minimum is attainable.

For \textcircled{2}, we only need to show that when one column of $\bm{X}$ is not correctly clustered, $\sum_{j=1}^k\Vert\bm{C}^{(j)}\Vert_{2,0}\geq n+1$. Without loss of generality, we assume that $\bm{x}_1\in\mathcal{S}_j$, $\bm{x}_2\in\mathcal{S}_l$, and $j\neq l$. Suppose that $\bm{x}_1$ and $\bm{x}_2$ are assigned into $\mathcal{S}_p$ corresponding to $\bm{D}^{(p)}$, where $p\in[k]$. Since the subspaces are independent and $\min_{j\in[d],i\in[n]}\vert z_{ji}\vert>0$, to ensure there exist some $\bm{c}_1$ and $\bm{c}_2$ such that $\bm{x}_1=\bm{D}^{(p)}\bm{c}_1$ and $\bm{x}_2=\bm{D}^{(p)}\bm{c}_2$, the column space of $\bm{D}^{(p)}$ must contain $\bm{U}^{(j)}$ and $\bm{U}^{(l)}$. It indicates $\hat{d}\geq 2d$, which is contradiction to the assumption $\hat{d}<2d$. Hence, at least one column of $[\bm{U}^{(j)}, \bm{U}^{(l)}]$ is contained in the column space of some $\bm{D}^{(q)}$, where $q\neq l\neq j$. As a result, 
\begin{equation*}
\bm{x}_1=[\bm{D}^{(p)},\bm{D}^{(q)}]
\left[
\begin{matrix}
\bm{c}_1^{(p)}\\
\bm{c}_1^{(q)}
\end{matrix}
\right]
\quad \text{or} \quad
\bm{x}_2=[\bm{D}^{(p)},\bm{D}^{(q)}]
\left[
\begin{matrix}
\bm{c}_2^{(p)}\\
\bm{c}_2^{(q)}
\end{matrix}
\right]
\end{equation*}
where $\bm{c}_1^{(p)},\bm{c}_1^{(q)}\neq\bm{0}$ or $\bm{c}_2^{(p)},\bm{c}_2^{(q)}\neq\bm{0}$. Therefore, $\sum_{j=1}^k\Vert\bm{C}^{(j)}\Vert_{2,0}\geq n+1$, if the data are not correctly clustered. In other words, if $\sum_{j=1}^k\Vert\bm{C}^{(j)}\Vert_{2,0}<n+1$, all columns of $\bm{X}$ are clustered correctly. Together with \textcircled{1}, we finish the proof.

(b) The proof is similar to that for (a) and is omitted for simplicity.
\end{proof}

\section{Proof for Theorem \ref{the_d}}\label{app_th1}
Before proving the theorem, we give the following lemmas. 
\begin{lemma}\label{lem_opt1sparse}
Let $\lbrace \hat{\bm{C}},\hat{\bm{D}}\rbrace$ be the optimal solution of \eqref{eq.ksc_2_0}. Then for all $i\in[n]$, $\sum_{j=1}^k\mathbb{1}(\Vert\hat{\bm{C}}_{:i}^{(j)}\Vert\neq\bm{0})=1$.
\end{lemma}
\begin{proof}
Suppose $\bm{x}=\sum_{j=1}^k\bm{D}^{(j)}\bm{c}^{(j)}$ and $\bm{x}=\bm{D}^{(l)}\bm{\alpha}$ where $l\in[k]$. It follows that
\begin{equation}\label{eq_proof_ca}
\bm{D}^{(l)}\bm{\alpha}=\sum_{j=1}^k\bm{D}^{(j)}\bm{c}^{(j)}.
\end{equation}
The minimum-norm solution of $\bm{\alpha}$ in \eqref{eq_proof_ca} is
\begin{equation}
\hat{\bm{\alpha}}={{\bm{D}}^{(l)}}^{\dagger} \sum_{j=1}^k\bm{D}^{(j)}\bm{c}^{(j)},
\end{equation}
where ${{\bm{D}}^{(l)}}^\dagger=\bm{V}\left[\begin{matrix}\bm{\Sigma}^{-1}& \bm{0}\\\bm{0}& \bm{0}\end{matrix}\right]\bm{U}^\top$ denotes the Moore–Penrose inverse of $\hat{\bm{D}}^{(l)}$ and $\bm{V},\bm{\Sigma},\bm{U}$ are from the SVD $\bm{D}^{(l)}=\bm{U}\left[\begin{matrix}\bm{\Sigma}& \bm{0}\\\bm{0}& \bm{0}\end{matrix}\right]\bm{V}^\top$. 
We have
\begin{equation}\label{eq.akc}
\begin{aligned}
\Vert \hat{\bm{\alpha}}\Vert=&\Vert\sum_{j=1}^k{{\bm{D}}^{(l)}}^{\dagger}\bm{D}^{(j)}\bm{c}^{(j)}\Vert \leq \Vert\bm{c}^{(l)}\Vert+\sum_{j\neq l}\Vert{{\bm{D}}^{(l)}}^{\dagger}\bm{D}^{(j)}\bm{c}^{(j)}\Vert \\
\leq&\Vert\bm{c}^{(l)}\Vert+\sum_{j\neq l}\Vert{{\bm{D}}^{(l)}}^{\dagger}\bm{D}^{(j)}\Vert_2\Vert\bm{c}^{(j)}\Vert\leq \sum_{j=1}^k\Vert\bm{c}^{(j)}\Vert,
\end{aligned}
\end{equation}
where the first inequality used the fact ${{\bm{D}}^{(l)}}^{\dagger}\bm{D}^{(l)}=\bm{I}$ and the third inequality used the condition in $\mathbb{S}_{D}^+$. In \eqref{eq.akc}, if $\bm{c}^{(j)}=\bm{0}$ for all $j\neq l$, the equality holds and then $\hat{\bm{c}}=\hat{\bm{\alpha}}$. Because $\bm{x}=\bm{D}\hat{\bm{\alpha}}$, $\hat{\bm{\alpha}}\neq\bm{0}$. Now expanding the result to all columns of $\bm{X}$, we finish the proof.
\end{proof}

\begin{lemma}\label{lem_aaa}
Suppose $\bm{x}\in\mathcal{S}_\ell$ and  $\bm{x}=\bm{D}^{(l)}\bm{\alpha}$. Denote $\bm{U}^{(\ell)}$ the basis of $\mathcal{S}_\ell$. The minimum of $\Vert\bm{\alpha}\Vert$ is not attained if some columns of $\bm{D}^{(l)}$ are not in $\text{span}(\bm{U}^{(\ell)})$. 
\end{lemma}

\begin{proof}
 We partition $\bm{D}^{(l)}$ into two parts
$\bm{D}^{(l)}=[\bm{D}^{(l)}_\ell\ \ \bm{D}^{(l)}_{-\ell}]$,
where ${\bm{D}}_\ell^{(l)}\in\mathbb{R}^{m\times d_u}$ , ${\bm{D}}_{-\ell}^{(l)}\in\mathbb{R}^{m\times (\hat{d}-d_u)}$, and $1\leq d_u\leq \hat{d}-1$. The columns of ${\bm{D}}_{-\ell}^{(l)}$ are not in $\text{span}(\bold{U}^{(\ell)})$. A smaller $\Vert\bm{\alpha}\Vert$ is obtained when ${\bm{\alpha}}\leftarrow[\alpha_{1},\ldots,\alpha_{d_u},0,\ldots,0]^\top$.

Let $\bar{\bm{D}}^{(l)}=[\bm{D}^{(l)}_\ell\ \ \tilde{\bm{D}}_{\ell}^{(l)}]$,
where $\tilde{\bm{D}}_{\ell}^{(j)}\in\mathbb{R}^{m\times (\hat{d}-d_u)}$ and the columns of $\tilde{\bm{D}}_{\ell}^{(j)}$ are in $\text{span}(\bm{U}^{(\ell)})$. There is always a $\tilde{\bm{D}}_{\ell}^{(l)}$ such that
\begin{equation}
\Vert\bm{\alpha}\Vert>\Vert\bar{\bm{\alpha}}\Vert,
\end{equation}
where $\bar{\bm{D}}^{(l)}\bar{\bm{\alpha}}={\bm{D}}^{(l)}\bm{\alpha}$. An example is 
$\tilde{\bm{D}}_{\ell}^{(l)}=[\bm{\delta},\ldots,\bm{\delta}]$,
where $\bm{\delta}$ is the last column of $\bm{D}^{(l)}_\ell$. Accordingly,
$$
\bar{\bm{\alpha}}=[\alpha_{1},\ldots,\alpha_{d_u-1},\tfrac{\alpha_{d_u}}{\hat{d}-d_u+1},\ldots,\tfrac{\alpha_{d_u}}{\hat{d}-d_u+1}]
.$$
Obviously, $\Vert\bm{\alpha}\Vert>\sqrt{\sum_{i=1}^{d_u-1}{\alpha_{i}}^2+\frac{{\alpha_{d_u}}^2}{\hat{d}-d_u+1}}=\Vert\bar{\bm{\alpha}}\Vert$.
\end{proof}

Now combining Lemma \ref{lem_opt1sparse} and Lemma \ref{lem_aaa}, we conclude that in the optimal solution of \eqref{eq.ksc_2_0},  the columns of each $\bm{D}^{(j)}$ are in the span of one subspace's bases, each column of $\bm{X}$ is reconstructed by only one sub-matrix of $\bm{D}$, and cannot be reconstructed by an incorrect sub-matrix (since $\min_{j\in[d],i\in[n]}\vert z_{ji}\vert>0$). This finished the proof.

\section{Proof for Theorem \ref{the_lambda}}
\begin{proof}
We have the following result.
\begin{lemma}[\cite{haltmeier2013block}]\label{lem_subgrad}
The subgradient of $\ell_{2,1}$ norm is
\begin{equation}
\partial \Vert \bm{x}\Vert_{2,1}=\left\{
\begin{array}{ll}
\bm{x}/\Vert \bm{x}\Vert, &\textup{if}\ \Vert\bm{x}\Vert>0;\\
\bm{z}: \Vert \bm{z}\Vert<1, &\textup{if} \ \Vert\bm{x}\Vert=0.
\end{array}
\right.
\end{equation} 
\end{lemma}
The optimality for the problem in the proposition indicates that 
$${\bm{D}^{(j)}}^\top(\bm{x}-\bm{D}\bm{c}^{(j)})=\lambda\partial \Vert \bm{c}^{(j)}\Vert_{2,1},\quad \forall j\in[k].$$
Letting $\bm{c}=0$ be the optimal solution, we have
$$ {\Vert\bm{D}^{(j)}}^\top\bm{x}\Vert< \lambda,\quad \forall j\in[k].$$
It means $\lambda>\max_j\Vert {\bm{D}^{(j)}}^\top\bm{x}\Vert$.
Expanding the result for all columns of $\bm{X}$, we finish the proof.
\end{proof}

\section{Proof for Theorem \ref{the_convergence}}
\begin{proof}
First, we give the following two lemmas.
\begin{lemma}[Lemma 10.4 in \cite{beck2017first}]\label{lem_D_inner}
Denote $\mathcal{L}(\bm{D}_{t_u})=\tfrac{1}{2}\Vert \bm{X}-\bm{D}_{t_u}\bm{C}_t\Vert_F^2$, where $\bm{D}_{t_u}\in\mathbb{S}_D$. Then in Algorithm \ref{alg.updateD},
\begin{equation}
\mathcal{L}(\bm{D}_{t_{u-1}})-\mathcal{L}(\bm{D}_{t_u})\geq \dfrac{\kappa_t}{2}\Vert \bm{D}_{t_{u-1}}-\bm{D}_{t_{u}}\Vert_F^2.
\end{equation}
\end{lemma}

\begin{lemma}[Lemma 2.1 in \cite{xu2013block}]\label{lem_xu2.1}
Let $g(\bm{u})$ and $h(\bm{u})$ be two convex functions defined on the convex set $\mathcal{U}$ and $g(\bm{u})$ be differentiable. Let $f(\bm{u})=g(\bm{v})+h(\bm{u})$ and $\bm{u}^\ast=\mathop{\textup{argmin}}_{\bm{u}\in\mathcal{U}}\left\langle \nabla g(\bm{v}),\bm{u}-\bm{v} \right\rangle+\tfrac{L}{2}\Vert\bm{u}-\bm{v}\Vert^2+h(\bm{u})$. If
$$g(\bm{u}^\ast)\leq g(\bm{v})+\left\langle \nabla g(\bm{v}),\mathcal{P}_L(\bm{v})-\bm{v} \right\rangle+\tfrac{L}{2}\Vert\bm{u}^\ast-\bm{v}\Vert^2,$$
then for any $\bm{u}\in\mathcal{U}$ we have
$$ f(\bm{u})-f(\bm{u}^\ast)\geq\tfrac{L}{2}\Vert\bm{u}^\ast-\bm{v}\Vert^2+L\left\langle \bm{v}-\bm{u},\bm{u}^\ast-\bm{v} \right\rangle.$$
\end{lemma}
In Section \ref{sec_updateC}, we select $\tau_{j,t-1}$ to make
\begin{align*}
\mathcal{L}(\bm{C}^{(j)})\leq&\mathcal{L}(\hat{\bm{C}}^{(j)}_{t-1})+\left\langle \bm{C}^{(j)}-\hat{\bm{C}}_{t-1}^{(j)},\hat{\bm{G}}^{(j)}\right\rangle+\dfrac{\tau_{j,t-1}}{2}\Vert \bm{C}^{(j)}-\hat{\bm{C}}_{t-1}^{(j)}\Vert_F^2.
\end{align*}
Then use Lemma \ref{lem_xu2.1} and let $g=\mathcal{L}$, $h=\lambda\Vert\cdot\Vert_{2,1}$, $\bm{u}=\bm{C}^{(j)}_{t-1}$, and $\bm{v}=\hat{\bm{C}}^{(j)}_{t-1}$. Denote $\mathcal{F}_j(\bm{C}^{(j)}_{t})=\mathcal{L}(\bm{C}^{(j)}_{t})+\lambda\Vert\bm{C}^{(j)}_{t}\Vert_{2,1}$ and  $\mathcal{F}(\bm{C}_{t},\bm{D}_{t-1})=\mathcal{L}(\bm{C}_{t},\bm{D}_{t-1})+\lambda\sum_{j=1}^k\Vert\bm{C}^{(j)}_{t}\Vert_{2,1}$. We have
\begin{equation}
\begin{aligned}
&\mathcal{F}_j(\bm{C}^{(j)}_{t-1})-\mathcal{F}_j(\bm{C}^{(j)}_{t})\\
\geq & \dfrac{\tau_{j,t-1}}{2}\Vert \hat{\bm{C}}_{t-1}^{(j)}-{\bm{C}}_{t}^{(j)}\Vert_F^2+\tau_{j,t-1} \left\langle \hat{\bm{C}}_{t-1}^{(j)}-{\bm{C}}_{t-1}^{(j)}, {\bm{C}}_{t}^{(j)}-\hat{\bm{C}}_{t-1}^{(j)} \right\rangle\\
= & \dfrac{\tau_{j,t-1}}{2}\Vert {\bm{C}}_{t-1}^{(j)}-{\bm{C}}_{t}^{(j)}\Vert_F^2-\dfrac{\tau_{j,t-1}}{2}\eta_{j,t-1}^2 \Vert{\bm{C}}_{t-2}^{(j)}-{\bm{C}}_{t-1}^{(j)}\Vert_F^2\\
\geq & \dfrac{\tau_{j,t-1}}{2}\Vert {\bm{C}}_{t-1}^{(j)}-{\bm{C}}_{t}^{(j)}\Vert_F^2-\dfrac{\tau_{j,t-2}}{2}\psi(t-1)\delta^2 \Vert{\bm{C}}_{t-2}^{(j)}-{\bm{C}}_{t-1}^{(j)}\Vert_F^2,
\end{aligned}
\end{equation}
where $\psi(t-1)=0$ if $t\leq 2$ and $\psi(t-1)=1$ if $t> 2$, according to the setting of $\eta_{j,t-1}$ in Algorithm \ref{alg.k-FSC} in the main paper.

It follows that
\begin{equation}\label{eq.FF}
\begin{aligned}
&\mathcal{F}(\bm{C}_{t-1},\bm{D}_{t-1})-\mathcal{F}(\bm{C}_{t},\bm{D}_{t-1})
= \sum_{j=1}^k\mathcal{F}_j(\bm{C}^{(j)}_{t-1})-\mathcal{F}_j(\bm{C}^{(j)}_{t})\\
\geq & \sum_{j=1}^k\Big(\dfrac{\tau_{j,t-1}}{2}\Vert {\bm{C}}_{t-1}^{(j)}-{\bm{C}}_{t}^{(j)}\Vert_F^2-\dfrac{\tau_{j,t-2}}{2}\psi(t-1)\delta^2 \Vert{\bm{C}}_{t-2}^{(j)}-{\bm{C}}_{t-1}^{(j)}\Vert_F^2\Big).
\end{aligned}
\end{equation}
On the other hand, according to Lemma \ref{lem_D_inner}, we have
\begin{equation}\label{eq.DD}
\begin{aligned}
&\mathcal{F}(\bm{C}_{t},\bm{D}_{t-1})-\mathcal{F}(\bm{C}_{t},\bm{D}_{t})
= \sum_{u=1}^{\vartheta}\dfrac{\kappa_t}{2}\Vert \bm{D}_{t_{u-1}}-\bm{D}_{t_{u}}\Vert_F^2\\
= & \dfrac{\kappa_t}{2}\Vert \bm{D}_{t-1}-\bm{D}_{t_1}\Vert_F^2+\dfrac{\kappa_t}{2}\Vert \bm{D}_{t_{\vartheta-1}}-{\bm{D}}_{t}\Vert_F^2+\Delta_{D_t}\triangleq \tilde{\Delta}_{D_t},
\end{aligned}
\end{equation}
where $\Delta_{D_t}=\sum_{u=2}^{\vartheta-1}\dfrac{\kappa_t}{2}\Vert \bm{D}_{t_{u-1}}-\bm{D}_{t_{u}}\Vert_F^2$.

Combining \eqref{eq.FF} with \eqref{eq.DD}, we have
\begin{equation}\label{eq.FFall}
\begin{aligned}
&\mathcal{F}(\bm{C}_{t-1},\bm{D}_{t-1})-\mathcal{F}(\bm{C}_{t},\bm{D}_{t})\geq  \sum_{j=1}^k\Big(\dfrac{\tau_{j,t-1}}{2}\Vert {\bm{C}}_{t-1}^{(j)}-{\bm{C}}_{t}^{(j)}\Vert_F^2\\
&\quad-\dfrac{\tau_{j,t-2}}{2}\psi(t-1)\delta^2 \Vert{\bm{C}}_{t-2}^{(j)}-{\bm{C}}_{t-1}^{(j)}\Vert_F^2\Big)+ \tilde{\Delta}_{D_t}.
\end{aligned}
\end{equation}

Now summing \eqref{eq.FFall} over $t$ from 1 to $T$, we arrive at
\begin{equation}\label{eq.dif_obj}
\begin{aligned}
&\mathcal{F}(\bm{C}_{0},\bm{D}_0)-\mathcal{F}(\bm{C}_{T},\bm{D}_T)\\
\geq & \sum_{t=1}^T\sum_{j=1}^k\Big(\dfrac{\tau_{j,t-1}}{2}\Vert {\bm{C}}_{t-1}^{(j)}-{\bm{C}}_{t}^{(j)}\Vert_F^2\\
&\qquad-\dfrac{\tau_{j,t-2}}{2}\psi(t-1)\delta^2 \Vert{\bm{C}}_{t-2}^{(j)}-{\bm{C}}_{t-1}^{(j)}\Vert_F^2\Big)+ \sum_{t=1}^T \tilde{\Delta}_{D_t}\\
= & \sum_{j=1}^k\dfrac{\tau_{j,0}}{2}\Vert {\bm{C}}_{0}^{(j)}-{\bm{C}}_{1}^{(j)}\Vert_F^2+\sum_{j=1}^k\dfrac{\tau_{j,T-1}}{2}\Vert {\bm{C}}_{T-1}^{(j)}-{\bm{C}}_{T}^{(j)}\Vert_F^2\\
&+\sum_{t=2}^T\sum_{j=1}^k\dfrac{(1-\delta^2)\tau_{j,t-1}}{2}\Vert {\bm{C}}_{t-1}^{(j)}-{\bm{C}}_{t}^{(j)}\Vert_F^2+ \sum_{t=1}^T \tilde{\Delta}_{D_t}\\
\geq  & \sum_{j=1}^k\dfrac{(1-\delta^2)\tau_{j,0}}{2}\Vert {\bm{C}}_{0}^{(j)}-{\bm{C}}_{1}^{(j)}\Vert_F^2\\
& +\sum_{j=1}^k\dfrac{(1-\delta^2)\tau_{j,T-1}}{2}\Vert {\bm{C}}_{T-1}^{(j)}-{\bm{C}}_{T}^{(j)}\Vert_F^2\\
&+\sum_{t=2}^T\sum_{j=1}^k\dfrac{(1-\delta^2)\tau_{j,t-1}}{2}\Vert {\bm{C}}_{t-1}^{(j)}-{\bm{C}}_{t}^{(j)}\Vert_F^2+ \sum_{t=1}^T \tilde{\Delta}_{D_t}\\
\geq & \sum_{t=1}^T\dfrac{(1-\delta^2)\bar{\tau}}{2}\Vert {\bm{C}}_{t-1}-{\bm{C}}_{t}\Vert_F^2+ \sum_{t=1}^T \tilde{\Delta}_{D_t},
\end{aligned}
\end{equation}
where $\bar{\tau}=\min\lbrace\tau_{j,t}: j=1,\ldots,k,\ t=0,\ldots,T\rbrace$. Notice that $\bar{\tau}>0$ according to its definition.
Since $\mathcal{F}(\bm{C},\bm{D})$ is bounded blow and  the two parts in the right-hand-side of the last inequality of \eqref{eq.dif_obj} are nonnegative, letting $T\rightarrow \infty$, we have
$$\sum_{t=1}^\infty\dfrac{(1-\delta^2)\bar{\tau}}{2}\Vert {\bm{C}}_{t-1}-{\bm{C}}_{t}\Vert_F^2<\infty$$
$$\sum_{t=1}^\infty \dfrac{\kappa_t}{2}\Vert \bm{D}_{t-1}-\bm{D}_{t_1}\Vert_F^2+\dfrac{\kappa_t}{2}\Vert \bm{D}_{t_{\vartheta-1}}-{\bm{D}}_{t}\Vert_F^2+\Delta_{D_t}<\infty.$$
It follows that
\begin{equation*}
\begin{aligned}
&\lim_{t\rightarrow\infty}\Vert \bm{C}_{t-1}-\bm{C}_{t}\Vert_F=0,\quad \lim_{t\rightarrow\infty}\Vert \bm{D}_{t-1}-\bm{D}_{t}\Vert_F=0,\\
&\lim_{t\rightarrow \infty} \mathcal{F}(\bm{C}_{t-1},\bm{D}_{t-1})-\mathcal{F}(\bm{C}_t,\bm{D}_t)= 0.
\end{aligned}
\end{equation*}
\end{proof}

\end{document}